\documentclass[12pt]{article}

\usepackage{overpic}
\usepackage{rotating,cite}

\usepackage[colorlinks, linkcolor=blue, citecolor=blue]{hyperref}
\usepackage{color}
\usepackage{graphicx,subfigure,amsmath,amssymb,amsfonts,bm,epsfig,epsf,url,dsfont}
\usepackage{amsthm}
\usepackage{tikz}
\usepackage{bbm}      
\usepackage{booktabs}
\usepackage{cases}
\usepackage{fullpage,mathrsfs}
\usepackage[small,bf]{caption}

\usepackage[top=1in,bottom=1in,left=1in,right=1in]{geometry}
\usepackage{fancybox}

\usepackage{algorithm}
\usepackage{verbatim}
\usepackage{algorithmic}

\usepackage{mathtools}

\usepackage{scalerel,stackengine}
\stackMath
\newcommand\reallywidehat[1]{%
\savestack{\tmpbox}{\stretchto{%
  \scaleto{%
    \scalerel*[\widthof{\ensuremath{#1}}]{\kern-.6pt\bigwedge\kern-.6pt}%
    {\rule[-\textheight/2]{1ex}{\textheight}}
  }{\textheight}%
}{0.5ex}}%
\stackon[1pt]{#1}{\tmpbox}%
}

\newcommand{\bE}{\mathbb{E}}

\newcommand{\Var}{\text{Var}}

\newtheorem{definition}{Definition}
\newtheorem{example}{Example}

\newtheorem{theorem}{Theorem}

\newtheorem{conjecture}{Conjecture}
\newtheorem{lemma}{Lemma}

\newtheorem{rmk}{Remark}

\newenvironment{fminipage}%
  {\begin{Sbox}\begin{minipage}}%
  {\end{minipage}\end{Sbox}\fbox{\TheSbox}}

\makeatletter
\newcommand*{\rom}[1]{\expandafter\@slowromancap\romannumeral #1@}
\makeatother

\newcommand{\Cov}{\mathrm{Cov}}

\newcommand{\Ind}{\mathbbm{1}}

\newcommand{\abs}[1]{\left|#1\right|}

\newcommand{\E}{\mathbb{E}}

\def\P{{\mathbb P}}
\def\S{{\mathbb S}}

\newcommand {\pr} {\mathbb{P}}

\newcommand{\calB}{{\cal B}}

\newcommand{\calD}{{\cal D}}
\newcommand{\calE}{{\cal E}}
\newcommand{\calF}{{\cal F}}
\newcommand{\calG}{{\cal G}}
\newcommand{\calH}{{\cal H}}

\newcommand{\calL}{{\cal L}}

\newcommand{\calN}{{\cal N}}
\newcommand{\calO}{{\cal O}}
\newcommand{\calP}{{\cal P}}
\newcommand{\calQ}{{\cal Q}}

\newcommand{\calS}{{\cal S}}
\newcommand{\calT}{{\cal T}}

\newcommand{\calX}{{\cal X}}
\newcommand{\calY}{{\cal Y}}

\newcommand{\be}{\begin{equation}}
\newcommand{\ee}{\end{equation}}
\newcommand{\beqna}{\begin{eqnarray}}
\newcommand{\eeqna}{\end{eqnarray}}
\newcommand{\inprod}[1]{\langle #1 \rangle}

\DeclarePairedDelimiterX{\set}[1]{\{}{\}}{\setargs{#1}}
\DeclarePairedDelimiterX{\cond}[1]{[}{]}{\setargs{#1}}
\NewDocumentCommand{\setargs}{>{\SplitArgument{1}{;}}m}
{\setargsaux#1}
\NewDocumentCommand{\setargsaux}{mm}
{\IfNoValueTF{#2}{#1} {#1\,\delimsize|\,\mathopen{}#2}}

\newcommand{\indep}{\perp \!\!\! \perp}

\newcommand{\p}[1]{\left(#1\right)}
\newcommand{\pp}[1]{\left[#1\right]}
\newcommand{\ppp}[1]{\left\{#1\right\}}
\newcommand{\norm}[1]{\left\|#1\right\|}

\usepackage{xspace}

\newcommand{\s}[1]{\mathsf{#1}}
\makeatletter
\def\thanks#1{\protected@xdef\@thanks{\@thanks
        \protect\footnotetext{#1}}}
\makeatother

\allowdisplaybreaks

\begin{document}
\title{Testing Dependency of Weighted Random Graphs}
\author{Mor~Oren-Loberman~~~~~~~Vered~Paslev~~~~~~~Wasim Huleihel\thanks{M. Oren, V. Paslev, and  W. Huleihel are with the Department of Electrical Engineering-Systems at Tel Aviv University, {T}el {A}viv 6997801, Israel (e-mails:  \texttt{orenmor@mail.tau.ac.il,veredpaslev@mail.tau.ac.il, wasimh@tauex.tau.ac.il}). This work is supported by the ISRAEL SCIENCE FOUNDATION (grant No. 1734/21).}}

\maketitle

\begin{abstract}
In this paper, we study the task of detecting the edge dependency between two weighted random graphs. We formulate this task as a simple hypothesis testing problem, where under the null hypothesis, the two observed graphs are statistically independent, whereas under the alternative, the edges of one graph are dependent on the edges of a uniformly and randomly vertex-permuted version of the other graph. For general edge-weight distributions, we establish thresholds at which optimal testing becomes information-theoretically possible or impossible, as a function of the total number of nodes in the observed graphs and the generative distributions of the weights. Finally, we identify a statistical-computational gap, and present evidence suggesting that this gap is inherent using the framework of low-degree polynomials. 
\end{abstract}

\section{Introduction}
\allowdisplaybreaks
Consider the following decision problem. We observe two weighted random graphs that are either generated independently at random or are edge-dependent due to some latent vertex correspondence or permutation. For this basic problem, 
two natural questions arise: the detection problem, which concerns whether the graphs exhibit dependence, and the recovery problem, which concerns identifying the latent correspondence between vertices. Here, we address the former question, specifically, we aim to understand under what conditions, in terms of the number of vertices and the generative distributions, one can distinguish between the two hypotheses and detect whether these graphs are dependent or not, say, with high probability?

The fundamental question above was first introduced and analyzed in \cite{YihongJiamingSophie}, where for Gaussian-weighted and dense Erd\H{o}s-R\'{e}nyi random graphs on $n$ vertices, sharp information-theoretic thresholds were developed, revealing the exact barrier at which the asymptotic optimal detection error probability undergoes a phase transition from zero to one as $n$ approaches infinity. For sparse Erd\H{o}s-R\'{e}nyi random graphs this threshold was initially determined within a constant factor in the same paper. Later on, this result was refined in \cite{DingDu}, where a sharp information-theoretic threshold was established for a specific range of edge probabilities. Our work is closely related to these works, and much inspired by \cite{YihongJiamingSophie}; we study the detection limit under general weight distributions, beyond the Erd\H{o}s-R\'{e}nyi and Gaussian Wigner graphs models mentioned above, and thereby aim to answer \cite[Sec. 6, Problem 2]{YihongJiamingSophie}. 

In particular, for a given pair of generative weight distributions (under the null and alternative hypotheses) we derive information-theoretic lower and upper bounds on the phase transition at which optimal testing error probability jumps from zero to one, as $n\to\infty$, and show that under mild assumptions on the weight distributions, arising naturally from classical binary hypothesis testing (but include a wide range of distributions), these bounds are sharp. We also consider testing in polynomial time; we propose and analyze a statistically suboptimal but polynomial-time efficient algorithm. In a similar vein to \cite{ding2023polynomialtime}, we observe a statistical-computational gap in our problem, which we suspect is fundamental. We provide evidence that any polynomial-time algorithm will fail in the region where the information-theoretic exhaustive algorithm is successful and the above polynomial-time algorithm fails.

\paragraph{Related work.} Recently, there has been a growing interest in the problem of random graph matching, also known as network alignment problem on the detection of two random graphs, and more generally, two databases, as well as the closely related problem of matching the vertex correspondence in the presence of correlation. This research is driven by a broad spectrum of practical applications, covering areas such as computational biology \cite{pmid:18725631,kang2012fast}, social network analysis, data anonymization and privacy-focused systems \cite{4531148,5207644}, and computer vision \cite{Berg2005,10.5555/2976456.2976496}, among others. In addition to \cite{YihongJiamingSophie}, which is most related to our paper, we mention a few other related works briefly. The information-theoretic thresholds for both the problems of detection and matching (recovery) are fairly well-understood under the Erd\H{o}s-R\'{e}nyi and Gaussian Wigner probabilistic models. For detection, significant progress has been made in establishing sharp thresholds and understanding the fundamental limits of exact recovery, as shown in works such as \cite{cullina2016improved, cullina2017exact, YihongJiamingSophie, GraphAl6}. For matching and recovery, advancements include achieving polynomial-time algorithms for graph similarity, identifying information-theoretic thresholds for vertex matching, and developing methods for partial recovery and alignment, as demonstrated in \cite{Hall, Ganassali, barak2019nearly, Ding,Ding2022MatchingRT}.
Furthermore, from an algorithmic perspective, a variety of polynomial-time efficient algorithms have been proposed along with rigorous performance guarantees. These include approaches for graph alignment and matching, such as those based on message-passing techniques, likelihood ratio tests, and spectral methods. Notable advances include algorithms for exact recovery in correlated Erd\H{o}s-R\'{e}nyi graphs
graphs, matching in Gaussian Wigner models, and partial alignment in sparse random graphs, as demonstrated in \cite{Ganassali,ganassali2023statistical,Cmao,GraphAl5,GraphAl4,ChengOtter,ding2022polynomial,ding2023polynomialtime,fan2023spectral}. These algorithms have been essential in providing practical solutions to the graph matching problem, although they often fall short of the theoretical optimum. Interestingly, current belief is that there is a statistical-computational gap in the detection and recovery problems, in the sense that state-of-the-art efficient algorithms thresholds are strictly larger than the information-theoretic ones, and as so there is a statistical cost for computational efficiency. Quite recently, an evidence for this gap was given in \cite{ding2023lowdegree} for Erd\H{o}s-R\'{e}nyi graphs based on the low-degree polynomial conjecture. It was demonstrated that low-degree polynomial algorithms fail for detection in both dense and sparse regimes, with specific results showing that any degree-$d$ polynomial algorithm fails for detection when the edge correlation is above a certain threshold, and in the sparse regime, when the edge density is sufficiently low. This suggests that several state-of-the-art algorithms for correlation detection and exact matching recovery may be close to optimal, reinforcing the statistical-computational gap. 
Finally, we mention here that the network alignment problem is closely related to a wide variety of other planted matching problems, both conceptually and technically, e.g., the database alignment and detection problem, with many exciting and interesting results, see, e.g., \cite{10.1109/ISIT.2018.8437908,pmlr-v89-dai19b,9834731,tamir,HuleihelElimelech,paslev2023testing}, and many references therein.


\paragraph{Notation.} For any $n\in\mathbb{N}$, the set of integers $\{1,2,\dots,n\}$ is denoted by $[n]$, and we define $n^{(2)}\triangleq\frac{n(n-1)}{2}$. Let $\mathbb{S}_n$ denote the set of all permutations on $[n]$. For a given permutation $\pi\in\mathbb{S}_n$, let $\pi_i$ denote the value to which $\pi$ maps $i\in[n]$. 
We use $\calN(\eta,\Sigma)$ to represent the multivariate normal random vector with mean vector $\eta$ and covariance matrix $\Sigma$, and $\s{Bern}(p)$ to denote a Bernoulli random variable with parameter $p$. Let $\calL(Y)$ denote the law, that is, the probability distribution, of a random variable $Y$. The variance of a random variable $X\sim\calP$ is denoted by $\s{Var}_{\calP}(X)$. For probability measures $\mathbb{P}$ and $\mathbb{Q}$, let $d_{\s{TV}}(\mathbb{P},\mathbb{Q})=\frac{1}{2}\intop |\mathrm{d}\mathbb{P}-\mathrm{d}\mathbb{Q}|$, $d_{\s{KL}}(\mathbb{P}||\mathbb{Q}) = \bE_{\mathbb{P}}\log\frac{\mathrm{d}\mathbb{P}}{\mathrm{d}\mathbb{Q}}$, $\chi^2(\mathbb{P}||\mathbb{Q}) = \bE_{\mathbb{Q}}(\frac{\mathrm{d}\mathbb{P}}{\mathrm{d}\mathbb{Q}}-1)^2 = \int\frac{\mathrm{d}\mathbb{P}^2}{\mathrm{d}\mathbb{Q}}-1$, and $\s{H}^2(\mathbb{P}||\mathbb{Q}) = \bE_{\mathbb{Q}}(1-\sqrt{\frac{\mathrm{d}\mathbb{P}}{\mathrm{d}\mathbb{Q}}})^2$, denote the total variation distance, the Kullback-Leibler (KL) divergence, the $\chi^2$-divergence, and squared Hellinger distance, respectively. 

\section{Problem Formulation}\label{sec:model}
\sloppy
\paragraph{Probabilistic model.} We follow the formulation of testing problem proposed in \cite{YihongJiamingSophie}, but generalize to arbitrary weighted random graphs. 
Specifically, let $G = ([n],\s{A})$ and $G' = ([n],\s{B})$ denote weighted undirected graphs on the vertex set $[n]$, with weighted adjacency matrices $\s{A}\in\mathbb{R}^{n\times n}$ and $\s{B}\in\mathbb{R}^{n\times n}$, respectively. We assume that the edge weights of each graph, namely, $\{\s{A}_{ij}\}_{1\leq i < j \leq n}$ and $\{\s{B}_{ij}\}_{1\leq i < j \leq n}$, to be i.i.d. Under the null hypothesis $\calH_0$, the matrices $\s{A}$ and $\s{B}$ are statistically independent, with $\s{A}_{i,j}\sim \calP_{{A}}$ and $\s{B}_{i,j}\sim \calP_{{B}}$, for all $i,j\in[n]$, respectively. It is assumed that $\calP_{{A}} = \calP_{{B}}$, and we denote the product measure $\calQ_{{A}{B}} = \calP_{{A}}\times \calP_{{B}}$. Under the alternative hypothesis $\calH_1$, given a latent random permutation $\pi\sim\s{Unif}(\mathbb{S}_n)$, the random variables $\{(\s{A}_{ij},\s{B}_{\pi_i\pi_j}): 1 \leq i < j \leq n\}$ are i.i.d., and each pair $(\s{A}_{ij},\s{B}_{\pi_i\pi_j})\sim \calP_{{A}{B}}$, with the same marginals as $\calQ_{{A}{B}}$, i.e., $\calP_{{A}}=\calP_{{B}}$. Upon observing $\s{A}$ and $\s{B}$, the objective is to evaluate $\calH_0$ against $\calH_1$. 
    For clarity and brevity, we will denote univariate random variables by $A$ and $B$ and adjacency matrices by $\s{A}$ and $\s{B}$ throughout this paper.
To summarize, the hypothesis testing problem we deal with is,
\begin{equation}\label{eqn:testing}
\begin{aligned}
    &\calH_0: (\s{A}_{ij},\s{B}_{ij})\stackrel{\mathrm{i.i.d}}{\sim} \calQ_{{AB}}\\
& \calH_1: (\s{A}_{ij},\s{B}_{\pi_i\pi_j})\stackrel{\mathrm{i.i.d}}{\sim} \calP_{{AB}},\ \ \text{conditional on } \pi\sim\s{Unif}(\mathbb{S}_n).
\end{aligned}
\end{equation}

\paragraph{Learning problem.} A test function for our problem is a map $\phi:\{0,1\}^{n\times n}\times\{0,1\}^{n\times n}\to \{0,1\}$, designed to determine which hypothesis $\calH_0$ or $\calH_1$ occurred. "We characterize the risk associated with a test $\phi$ the total of its Type-I and II error probabilities, i.e.,
\begin{align}
\s{R}_n(\phi)\triangleq \P_{\calH_0}[\phi(\s{A},\s{B})=1]+\P_{\calH_1}[\phi(\s{A},\s{B})=0],
\end{align}
where $\pr_{\calH_0}$ and $\pr_{\calH_1}$ designate the probability distributions under the null and alternative hypotheses, respectively. Then, the optimal risk is,
\begin{align}
\s{R}_n^\star\triangleq\inf_{\phi:\{0,1\}^{n\times n}\times\{0,1\}^{n\times n}\to \{0,1\}}\s{R}_n(\phi).
\end{align}
We remark that $\s{R}$ is a function of $n$, however, we omit them from our notation for the benefit of readability. We study the possibility and impossibility of our detection problem in multiple asymptotic regimes. A test
$\phi$ achieves \emph{strong} detection if $\s{R}_n(\phi)$ approaches $0$ as $n$ tends to infinity, and \emph{weak} detection if $\limsup_{n\to\infty}\s{R}_n(\phi)<1$ (i.e.,  strictly outperforming a random guess). Our main goal is to investigate when strong/weak detection is possible and impossible, as a function of $n$ and the distributions $(\calP,\calQ)$.

\section{Main Results} 

In this section, we present our main results on the information-theoretic lower and upper bounds of weak and strong criteria of detection, under the setting presented in Section~\ref{sec:model}. Finally, we present our computational lower bounds. All proofs of our main results appear in Section~\ref{sec:proof}. We begin with our impossibility lower-bounds. 

\paragraph{Lower bounds.} Following \cite{YihongJiamingSophie,Paslev}, we introduce a few important notations. For any $x\in\calX$ and $y\in\calY$, we let $\calL(x,y) \triangleq \frac{\calP(x,y)}{\calQ(x,y)}$ denote the single-letter likelihood function. It turns our beneficial to think of $\calL$ as a kernel, which induces an operator as follows. For any square-integrable function $f$ under $\calQ$, i.e., $ \int\int f^2(x,y)\calQ(\mathrm{d}x)\calQ(\mathrm{d}y)<\infty$,
\begin{align}
    (\calL f)(x) \triangleq \bE_{Y\sim \calQ}\pp{\calL(x,Y)f(Y)}.\label{eqn:kernel}
\end{align}
Furthermore, the composition $\calL^2= \calL\circ\calL$ is given by $\calL^2(x,y) = \bE_{Z\sim Q}[\calL(x,Z)\calL(Z,y)]$, and $\calL^k$ is defined similarly. We assume that $\calL(x,y) = \calL(y,x)$, and hence $\calL$ is self-adjoint. Furthermore, if we assume that $ \int\int \calL^2(x,y)\calQ(\mathrm{d}x)\calQ(\mathrm{d}y)<\infty$, then $\calL$ is Hilbert-Schmidt. Thus, $\calL$ is diagonazable, and we denote its eigenvalues by $\{\lambda_i\}_{i\geq0}$. Accordingly, the trace of $\calL$ is given by $\s{trace}(\calL) = \bE_{Y\sim \calQ}[\calL(Y,Y)] = \sum_{i\in\mathbb{N}}\lambda_i$. Without loss of generality, we assume that the sequence of eigenvalues $\{\lambda_i\}_{i\geq0}$ is decreasing, namely, $\lambda_i\geq\lambda_{i+1}$, for all $i\in\mathbb{N}$. Finally, Lemma~\ref{lem:eigen} in Section~\ref{sec:proofLowerBounds} shows that the largest eigenvalue of $\calL$ is unity, i.e., $\lambda_0=1$. 

The log-likelihood ratio (LLR) is denoted by $\mathscr{L}(x,y)\triangleq\log\calL(x,y) = \log\frac{\calP(x,y)}{\calQ(x,y)}$, for $(x,y)\in\calX\times\calY$. To wit, the LLR is the logarithm of the Radon-Nikodym derivative between $\calP$ and $\calQ$. Without loss of essential generality, we assume $\calP$ and $\calQ$ are mutually absolutely continuous; therefore $\mathscr{L}$ is well-defined. We assume that $d_{\s{KL}}(\calP||\calQ)$ and $d_{\s{KL}}(\calQ||\calP)$, which are the expectations of the LLR w.r.t $\calP$ and $\calQ$, respectively, are finite. For $\theta\in(-d_{\s{KL}}(\calQ||\calP),d_{\s{KL}}(\calP||\calQ))$, the Chernoff's exponents $E_{P},E_{\calQ}:\mathbb{R}\to[-\infty,\infty)$ are,
\begin{align}
    E_{\calQ}(\theta)\triangleq\sup_{\lambda\in\mathbb{R}}\lambda\theta-\psi_{\calQ}(\lambda),\quad,E_{\calP}(\theta)\triangleq\sup_{\lambda\in\mathbb{R}}\lambda\theta-\psi_{\calP}(\lambda), \label{eq:Legendre_trans_def}
\end{align}
where $\psi_{\calQ}(\lambda)\triangleq\log\bE_{\calQ}[\exp(\lambda\mathscr{L})]$ and $\psi_{\calP}(\lambda)\triangleq\log\bE_{\calP}[\exp(\lambda\mathscr{L})]$. In other words, the Chernoff's exponents are the Legendre transforms of the log-moment generating functions. The following properties will be useful in our analysis. Note that $\psi_{\calP}(\lambda) = \psi_{\calQ}(\lambda+1)$, and thus $E_{\calP}(\theta) = E_{\calQ}(\theta)-\theta$. Furthermore, $E_{\calP}$ and $E_{\calQ}$ are convex non-negative functions. Since $\psi'_{\calQ}(0) = -d_{\s{KL}}(\calQ||\calP)$ and $\psi'_{\calQ}(1) = d_{\s{KL}}(\calP||\calQ)$, we obtain that $E_{\calQ}(-d_{\s{KL}}(\calQ||\calP)) = E_{\calP}(d_{\s{KL}}(\calP||\calQ))=0$, which in turn implies that $E_{\calQ}(d_{\s{KL}}(\calP||\calQ)) = d_{\s{KL}}(\calP||\calQ)$ and $E_{\calP}(-d_{\s{KL}}(\calQ||\calP))=d_{\s{KL}}(\calQ||\calP)$. Finally, it is well-known that $\lambda\theta-\psi_{\calQ}(\lambda)$ is concave, and its derivative at $\lambda=0$ is $\theta+d_{\s{KL}}(\calQ||\calP)$. Accordingly, the maximizer $\lambda^\star$ of the concave optimization in the definition of $E_{\calQ}(\theta)$ is non-negative if $\theta\geq - d_{\s{KL}}(\calQ||\calP)$. The same is true for $E_{\calP}(\theta)$, provided that $\theta\leq d_{\s{KL}}(\calP||\calQ)$. We are now in a position to state our main results.

\begin{theorem}[Weak detection lower bound]\label{th:2LB}
Consider the detection problem in \eqref{eqn:testing}, and fix $\alpha\in(0,1)$. Weak detection is statistically impossible, i.e., $\s{R}_n^\star=1-o(1)$, if 
\begin{subequations}
    \begin{align}
    &\chi^2\p{\calP || \calQ}\leq\frac{(2-\epsilon)\log n}{\alpha n},\quad\s{and}\label{eqn:condtrun1}\\
    &d_{\s{KL}}\p{\calP || \calQ} + \delta_n\cdot\s{Var}_{\calP}\p{\mathscr{L}}\leq\frac{(2-\epsilon)\log n}{n},\label{eqn:condtrun2}
\end{align} 
\end{subequations}
for any $\omega(1)=\delta_n=o(\log n)$, and any constant $\epsilon>0$.
\end{theorem}
\begin{rmk}\label{rmk:Th1}
It is well-known that $d_{\s{KL}}\p{\calP || \calQ}\leq\chi^2\p{\calP || \calQ}$, for any pair $(\calP,\calQ)$ (see, e.g., \cite[eq. (2.26)]{tsybakov2004introduction}), however, without any further assumptions on $(\calP,\calQ)$, an inequality of the form $\chi^2\p{\calP || \calQ}\leq \s{C}\cdot d_{\s{KL}}\p{\calP || \calQ}$, for some $\s{C}>1$ cannot hold. Nonetheless, it can be shown that the class of distributions for which there is a constant $\s{C}>1$ such that $\chi^2\p{\calP || \calQ}\leq \s{C}\cdot d_{\s{KL}}\p{\calP || \calQ}$ is non-trivial and interesting. For example, it can be shown that (see, e.g., \cite{Hajek17,pmlr-v99-brennan19a}):
\begin{enumerate}
    \item Pairs $(\calP,\calQ)$ with sub-Gaussian LLR, i.e., satisfying $\mathscr{L}({A},{B})$, for $(A,B)\sim\calQ$ is sub-Gaussian.
    \item Pairs $(\calP,\calQ)$ with bounded LLR, i.e., satisfying $\mathscr{L}({A},{B})$ for $({A},{B})\sim\calQ$ is bounded almost surely.
    \item A wide variety of other pairs $(\calP,\calQ)$ from a common exponential family.
\end{enumerate}
Furthermore, for this class of distributions, we have
\begin{align}
    \s{Var}_{\calP}\p{\mathscr{L}}&\leq \bE_\calP\mathscr{L}^2\leq \bE_\calP\pp{\frac{(\calL-1)^2}{\calL}}\\
    & = \chi^2\p{\calP || \calQ}\leq \s{C}\cdot d_{\s{KL}}\p{\calP || \calQ},
\end{align}
where the second inequality follows from the fact that $\log^2 x\leq\frac{(x-1)^2}{x}$, for $x\geq0$. Thus, in this case, the conditions in \eqref{eqn:condtrun1}--\eqref{eqn:condtrun2} boils down to
\begin{align}
    d_{\s{KL}}\p{\calP || \calQ}\leq\frac{(2-\epsilon)\log n}{n},
\end{align}
by taking $\alpha = \s{C}^{-1}$ and any sequence $\delta_n=o(\log n)$. We let $\calD_1$ denote pairs $(\calP,\calQ)$ satisfying $\chi^2\p{\calP || \calQ}\leq \s{C}\cdot d_{\s{KL}}\p{\calP || \calQ}$, for some $\s{C}>1$.
\end{rmk}
\begin{example}[Gaussian case]\label{exmp:1}
In the Gaussian case, we assume that $\calQ_{{A}{B}} = \calP_{{A}}\times \calP_{{B}}$, with $\calP_{{A}}$ and $\calP_{{B}}$ correspond to the densities of a standard normal random variable, whereas $\calP_{{A}{B}}$ is the joint density of two $\rho$-correlated zero-mean Gaussian random variables with unit variance, i.e.,
    \begin{align}
        \calP_{{A}{B}}\equiv\mathcal{N}\p{\begin{bmatrix}
0 \\
0 
\end{bmatrix},\begin{bmatrix}
1 & \rho\\
\rho & 1
\end{bmatrix}},
    \end{align}
for some known correlation coefficient $\rho\in[-1,1]\setminus\{0\}$. It can be shown that for this Gaussian setting we have $(\calP,\calQ)\in\calD_1$. Thus, because $d_{\s{KL}}(\calP||\calQ) = -\frac{1}{2}\log(1-\rho^2)$, and since $\rho=o(1)$, Remark~\ref{rmk:Th1} implies that weak detection is impossible if $\rho^2\leq \frac{(4-\epsilon)\log n}{n}$, recovering the result in \cite[Thm. 1]{YihongJiamingSophie}. 
\end{example}
\begin{example}[Bernoulli case]\label{exmp:2}
In the Bernoulli case, we assume that $\calQ_{{A}{B}} = \calP_{{A}}\times \calP_{{B}}$, with 
$\calP_{{A}}=\calP_{{B}} = \s{Bernoulli}(\tau p)$, for some $p\in(0,1)$ and $\tau\in[0,1]$, and $\calP_{{A}{B}}$ denotes the joint distribution of two correlated Bernoulli random variables. Specifically, under $\calP_{{A}{B}}$, we have $A\sim\s{Bernoulli}(\tau p)$, and
\begin{align}
    {B}\vert {A}&\sim\begin{cases}
    \s{Bernoulli}(\tau),\ &\s{if }\;{A}=1\\ 
    \s{Bernoulli}\p{\frac{\tau p(1-\tau)}{1-\tau p}},\ &\s{if }\;{A}=0.
    \end{cases}
\end{align}
Here, Pearson correlation coefficient is given by,
\begin{align}
    \rho\triangleq\frac{\s{cov}(A,B)}{\sqrt{\s{var}(A)}\sqrt{\s{var}(B)}} = \frac{\tau(1-p)}{1-\tau p}.
\end{align}
If $p = n^{-o(1)}$, then it can be shown that $(\calP,\calQ)\in\calD_1$, and since $d_{\s{KL}}(\calP||\calQ) = \tau^2p(\log p^{-1}-1+p)+o(\rho^2)$, Remark~\ref{rmk:Th1} implies that weak detection is impossible if $\tau^2p(\log p^{-1}-1+p)\leq \frac{(2-\epsilon)\log n}{n}$, recovering \cite[Eq. (6) in Thm. 2]{YihongJiamingSophie}. 
\end{example}

\paragraph{Upper bounds.} Let us present our detection algorithms and the corresponding upper bounds. Note that under $\calH_1$, the latent random permutation $\pi$ represents the hidden node correspondence under which $\s{A}$ and $\s{B}$ are correlated. For this reason, we refer to $\calH_1$ as the planted model and $\calH_0$ as the null model. As is well-known the optimal decision rule for \eqref{eqn:testing} is the Neyman-Pearson test which correspond to thresholding the likelihood ratio,
\begin{align}
    \s{L}_n(\s{A},\s{B})\triangleq\frac{\pr_{\calH_1}(\s{A},\s{B})}{\pr_{\calH_0}(\s{A},\s{B})}= \frac{1}{n!}\sum_{\pi\in\mathbb{S}_n}\frac{\pr_{\calH_1\vert\pi}(\s{A},\s{B}\vert\pi)}{\pr_{\calH_0}(\s{A},\s{B})}\triangleq\bE_{\pi\sim\s{Unif}(\mathbb{S}_n)}\pp{\s{L}_n(\s{A},\s{B}\vert\pi)}.
\end{align}
As is customary to other related testing problems with a latent combinatorial structure, analyzing the above optimal test is difficult due to the averaging over all $n!$ possible permutations. Consider, instead, the generalized likelihood ratio test (GLRT), where the averaging is replaced with a maximum, namely, 
\begin{align}
    \max_{\pi\in\mathbb{S}_n}\s{L}_n(\s{A},\s{B}\vert\pi) = \max_{\pi\in\mathbb{S}_n}\frac{\pr_{\calH_1\vert\pi}(\s{A},\s{B}\vert\pi)}{\pr_{\calH_0}(\s{A},\s{B})}.
\end{align}
Next, we derive a closed-form expression for the conditional likelihood $\s{L}_n(\s{A},\s{B}\vert\pi)$. We have
\begin{align}
    \s{L}_n(\s{A},\s{B}\vert\pi) = \prod_{(i,j)\in\Lambda}\calL(\s{A}_{ij},\s{B}_{\pi_i\pi_j}),
\end{align}
where $\Lambda\triangleq\{(i,j):1\leq i<j\leq n\}$.
For any $(i,j)\in\binom{[n]}{2}$, $\calQ$ denotes the joint density function of $\s{A}_{ij}$ and $\s{B}_{\pi_i\pi_j}$ under $\calH_0$, and
$\calP$ denotes the joint density function of $\s{A}_{ij}$ and $\s{B}_{\pi_i\pi_j}$ under $\calH_1$ given its latent permutation $\pi$. Accordingly, we define the test as,
\begin{align}
    \phi_{\s{GLRT}}(\s{A},\s{B})&\triangleq\Ind\ppp{\max_{\pi\in\mathbb{S}_n}\frac{1}{n^{(2)}}\sum_{1\leq i<j\leq n}\log \calL(\s{A}_{ij},\s{B}_{\pi_i\pi_j})\geq\tau_{\s{GLRT}}},\label{eqn:testscanmain}
\end{align}
for some threshold $\tau_{\s{GLRT}}\in\mathbb{R}$, and $n^{(2)}\triangleq\frac{n(n-1)}{2}$. We have the following result.
\begin{theorem}[GLRT strong detection]\label{thm:GLRT}
Consider the GLRT in \eqref{eqn:testscanmain}, and suppose there is a $\tau_{\s{GLRT}}\in(-d_{\s{KL}}(\calQ||\calP),d_{\s{KL}}(\calP||\calQ))$ with
\begin{subequations}\label{eqn:GLRTcond}
	\begin{align}
		E_{\calQ}(\tau_{\s{GLRT}})&\geq \frac{2\log (n/e)}{n-1}+\frac{2(1+\log n)}{n(n-1)}+\omega(n^{-2}) ,\label{eqn:GLRTcond1} \\
		E_{\calP}(\tau_{\s{GLRT}})&= \omega(n^{-2}) .\label{eqn:GLRTcond2}
	\end{align}
\end{subequations}
Then, $\s{R}(\phi_{\s{GLRT}})\to0$, as $n\to\infty$.
\end{theorem}
\begin{rmk} The bounds in \eqref{eqn:GLRTcond} can be simplified in many ways. For example, take $\tau_{\s{GLRT}}=0$, and then,
    \begin{align}
        E_{\calQ}(0) &\geq \sup_{\lambda\in[0,1]}-\psi_{\calQ}(\lambda)= \s{C}(\calP||\calQ)\label{eqn:Chernoff}\\
        &\geq -\log\pp{1-\frac{1}{2}\s{H}^2(\calP||\calQ)},
    \end{align}
where $\s{C}(\calP||\calQ)$ is the Chernoff information, and $\s{H}^2(\calP||\calQ)$ is the Hellinger distance. In the same way, we can get $E_{\calP}(0)\geq-\log\pp{1-\frac{1}{2}\s{H}^2(\calP||\calQ)}$ (consider the interval $[-1,0]$ in \eqref{eqn:Chernoff} instead). Thus, \eqref{eqn:GLRTcond} boils down to the single condition
\begin{align}
    \s{H}^2(\calP||\calQ)\geq \frac{4\log (n/e)}{n-1}+O(n^{-2}\log n).
\end{align}
Moreover, let $\calD_2$ denote the set of pairs of distributions $(\calP,\calQ)$ such that there is a constant $\s{C}>1$ such that for all $|\lambda|\leq1$,
\begin{subequations}\label{eqn:SubExpo}
	\begin{align}
		\psi_\calP(\lambda)&\leq d_{\s{KL}}\p{\calP || \calQ}\cdot\lambda+\s{C}\cdot d_{\s{KL}}\p{\calP || \calQ}\cdot \lambda^2,\label{eqn:SubExpo1} \\
		\psi_\calQ(\lambda)&\leq -d_{\s{KL}}\p{\calQ || \calP}\cdot\lambda+\s{C}\cdot d_{\s{KL}}\p{\calQ || \calP}\cdot \lambda^2.\label{eqn:SubExpo2}
	\end{align}
\end{subequations}
Again, it can be shown that $\calD_2$ is wide, in particular, include those distributions mentioned in Remark~\ref{rmk:Th1}. Furthermore, for $(\calP,\calQ)\in\calD_2$, it can be shown that $d_{\s{KL}}\p{\calP || \calQ}=\Theta\p{d_{\s{KL}}\p{\calQ || \calP}}$, and,
\begin{align}
   E_\calP(\tau_{\s{GLRT}})\geq \frac{(\theta - d_{\s{KL}}\p{\calP || \calQ})^2}{4\cdot\s{C}\cdot d_{\s{KL}}\p{\calP || \calQ}}; E_\calQ(\bar\theta)\geq \frac{(\theta+d_{\s{KL}}\p{\calQ || \calP})^2}{4\cdot\s{C}\cdot d_{\s{KL}}\p{\calQ || \calP}},
\end{align}
where $\theta=n^{(2)}\tau_{\s{GLRT}}$.
Thus, if we take $\tau_{\s{GLRT}} = 2\sqrt{\s{C}\cdot d_{\s{KL}}\p{\calP|| \calQ} \cdot d_{\s{KL}}\p{\calQ || \calP}} - d_{\s{KL}}\p{\calQ || \calP}$, we obtain $E_\calQ({\tau_{\s{GLRT}}})\geq d_{\s{KL}}\p{\calP || \calQ}$ and $E_\calP({\tau_{\s{GLRT}}})\geq \s{C}'\cdot d_{\s{KL}}\p{\calP || \calQ}$, for some constant $\s{C}'>0$. Then, the conditions in \eqref{eqn:GLRTcond1}--\eqref{eqn:GLRTcond2} reduce to $ d_{\s{KL}}\p{\calP || \calQ}\geq \frac{2\log n}{n-1}$. This complements the lower bound we obtained in Remark~\ref{rmk:Th1}, for $(\calP,\calQ)\in\calD_1$. Finally, it can be shown that for both the Gaussian and Bernoulli settings in Examples~\ref{exmp:1}--\ref{exmp:2}, we have $(\calP,\calQ)\in\calD_2$, and thus, GLRT archives strong detection if $\rho^2> \frac{4\log n}{n}$ and $\tau^2p(\log p^{-1}-1+p)> \frac{2\log n}{n}$, respectively.
\end{rmk}

The GLRT involves a combinatorial optimization that is intractable. As so, it is of practical interest to devise and analyze efficient algorithms. Consider the following simple test. Define $\s{corr}(\calP,\calQ)\triangleq\frac{\s{cov}_{\calP}({A},{B})}{\s{Var}_{\calQ}({A})}$ as the Pearson correlation coefficient between ${A}$ and ${B}$. For $\theta\in\mathbb{R}_+$, let,
\begin{align}
    \phi_{\s{comp}}(\s{A},\s{B})\triangleq\Ind\ppp{\abs{\sum_{(i,j)\in\Lambda}(\s{A}_{ij}-\s{B}_{ij})}\leq\theta},\label{eqn:sumtest}
\end{align}
if $\s{corr}(\calP,\calQ)\in(0,1]$, and flip the inequality in \eqref{eqn:sumtest} if $\s{corr}(\calP,\calQ)\in[-1,0)$. We have the following result.
\begin{theorem}[Comparison test weak detection]\label{thm:sumTest}
    Consider the comparison test in \eqref{eqn:sumtest}. Then, if $|\s{corr}(\calP,\calQ)| = \Omega(1)$, and 
    \begin{align}
        \frac{\bE_{\calQ}|{A}-{B}|^3}{\s{Var}^{3/2}_{\calQ}({A})},\frac{\bE_{\calP}|{A}-{B}|^3}{\s{Var}^{3/2}_{\calQ}({A})(1-|\s{corr}(\calP,\calQ)|)^{3/2}}= o(n),
    \end{align}
    then $\lim_{n\to\infty}\s{R}_n(\phi_{\s{sum}})<1$.
\end{theorem}
The performance guarantee in Theorem~\ref{thm:sumTest} is clearly very suboptimal compared to the one in Theorem~\ref{thm:GLRT}. For example, in the Gaussian setting in Example~\ref{exmp:1}, we have $\s{corr}(A,B) = \rho$, and so the comparison test is successful only if $|\rho| = \Omega(1)$, whereas the GLRT allows for a vanishing correlation. We suspect, however, that this is fundamental in the sense that this is a barrier for what can be achieved using polynomial-time algorithms, and we next give an evidence for this \emph{statistical-computational gap}, using the framework of low-degree polynomials (LDP).
\paragraph{Computational lower bound.} The premise of the LDP framework is that all polynomial-time algorithms for solving detection problems are captured by polynomials of low-degree. By now, there is growing strong evidence in support of this conjecture. The ideas below were developed in a sequence of works in the sum-of-squares optimization literature \cite{barak2016nearly,Hopkins18,hopkins2017bayesian,hopkins2017power}. 

With regard to the problem of detecting graph correlations, recently in \cite{ding2023lowdegree}, an evidence for such a gap was shown using this framework. For example, relying on the LDP conjecture, it was proved in \cite{ding2023lowdegree} that in the dense regime, if $\rho^2 = o(1)$, then polynomial tests with degree at most $O(\rho^{-1})$ fail for detection. Let us describe the basics of the LDP framework. We follow the notations and definition in \cite{Hopkins18,Dmitriy19}. Any distribution $\pr_{\calH_0}$ induces an inner product of measurable functions $f,g:\Omega_n\to\mathbb{R}$ given by $\left\langle f,g \right\rangle_{\calH_0} = \bE_{\calH_0}[f(\s{G})g(\s{G})]$, and norm $\norm{f}_{\calH_0} = \left\langle f,f \right\rangle_{\calH_0}^{1/2}$. Let $L^2(\pr_{\calH_0})$ be the Hilbert space of functions $f$ with $\norm{f}_{\calH_0}<\infty$, and endowed with the inner product above. 

The idea in the LDP framework is to find the low-degree polynomial that best distinguishes $\pr_{\calH_0}$ from $\pr_{\calH_1}$ in the $L^2$ sense. Let $\calL_{n,\leq\s{D}}\subset L^2(\pr_{\calH_0})$ denote the linear subspace of polynomials of degree at most $\s{D}\in\mathbb{N}$. 
Then, the \emph{$\s{D}$-low-degree likelihood ratio} $\s{L}_{n,\leq \s{D}}$ is the projection of a function $\s{L}_{n}$ to $\calL_{n,\leq\s{D}}$, where the projection is w.r.t. the inner product $\left\langle \cdot,\cdot \right\rangle_{\calH_0}$. As is well-known, the likelihood ratio is the optimal test to distinguish $\pr_{\calH_0}$ from $\pr_{\calH_1}$, in the $L^2$ sense. Accordingly, it can be shown that over $\calL_{n,\leq\s{D}}$, the function $\s{L}_{n,\leq \s{D}}$ exhibits the same property \cite{hopkins2017bayesian,hopkins2017power,Dmitriy19}. 
Furthermore, an important property of the likelihood ratio is that if $\norm{\s{L}_n}_{\calH_0}$ is bounded, then $\pr_{\calH_0}$ and $\pr_{\calH_1}$ are statistically indistinguishable. The following conjecture states that a computational analogue of this property holds. In fact it also postulates that polynomials of degree $\approx\log n$ are a proxy for polynomial-time algorithms. The conjecture below is based on \cite{Hopkins18,hopkins2017bayesian,hopkins2017power}, and \cite[Conj. 2.2.4]{Hopkins18}, and here we provide an informal statement which appears in \cite[Conj. 1.16]{Dmitriy19}. For a precise statement, see, e.g., \cite[Conj. 2.2.4]{Hopkins18} and \cite[Sec. 4]{Dmitriy19}.
\begin{conjecture}[Low-degree conj., informal]\label{conj:1}
Given a sequence of probability measures $\pr_{\calH_0}$ and $\pr_{\calH_1}$, if there exists $\epsilon>0$ and $\s{D} = \s{D}(n)\geq (\log n)^{1+\epsilon}$, such that $\norm{\s{L}_{n,\leq \s{D}}}_{\calH_0}$ remains bounded as $n\to\infty$, then there is no polynomial-time algorithm that distinguishes $\pr_{\calH_0}$ and $\pr_{\calH_1}$.
\end{conjecture}
In the sequel, we will rely on Conjecture~\ref{conj:1} to give an evidence for the statistical-computational gap observed in our problem.
For general distributions, it seems as though that further assumptions on $(\calP,\calQ)$ are needed, and we assume the following class of distributions. Specifically, let $\{\mathscr{P}_\ell\}_{\ell\geq0}$ be a sequence of orthonormal polynomials with respect to the marginal distributions of $\calP_{{AB}}$, such that 
$\mathscr{P}_{\ell}$ is a polynomial of order $\ell$. Consider the class of distributions of $\calP$, such that for each $\ell\in\mathbb{N}$, the projection $\mathbb{E}_{\calP}[{B}^{\ell}\vert {A}=a]$ is a polynomial of degree $\ell$ in $a$ with a leading coefficient $\kappa_{\ell}\le \abs{\s{corr}(\calP,\calQ)}^{\ell}$, where we recall that $\s{corr}(\calP,\calQ)$ denotes the Pearson correlation coefficient between ${A}$ and ${B}$. Note that since the marginal distributions of $\calP_{AB}$ are the same, then the projection $\mathbb{E}_{\calP}[{A}^{\ell}\vert {B}=b]$ is a polynomial of degree $\ell$ in $b$ as well. 
Then, for this class of distributions, we show in Lemma~\ref{lem:class3_prop}, that,
\begin{align}
    \mathbb{E}_{\calP}[\mathscr{P}_n({B})\vert {A}] = \kappa_n\mathscr{P}_n({A}).
\end{align}
This property turns out to be crucial in the derivation of the LDP computational lower bounds. The class of distributions $\calD_3$ above, include both the Gaussian and Bernoulli settings in Examples~\ref{exmp:1}--\ref{exmp:2}, but other distributions as well, such as, the family of stable distributions, Poisson, Chi-square, and possibility more; in Appendix~\ref{sec:Appendix} we show that these distribution indeed belong to $\calD_3$. As an example, in the Gaussian setting, under $\calP$, we have ${B} = \rho\cdot {A} + \sqrt{1-\rho^2}{W}$, with ${W}\sim\calN(0,1)$, and so 
$\E_{\calP}\pp{{B}^{\ell}\vert {A}=a}$ is an $\ell$-degree polynomial in $a$, with leading coefficient $\rho^\ell$. For this family of distributions, we can prove the following result.
\begin{theorem}\label{thm:gap_gen}
    Assume that $\calP\in\calD_3$. Then, if $|\s{corr}(\calP,\calQ)|= o(1)$, then $\norm{\s{L}_{n,\leq \s{D}}}_{\calH_0}\leq O(1)$, for any $\s{D}=O(|\s{corr}(\calP,\calQ)|^{-1/2})$. Conversely, there exists a positive integer $\s{D}$ such that if $|\s{corr}(\calP,\calQ)| = \Omega(1)$, then $\norm{\s{L}_{n,\leq \s{D}}}_{\calH_0}\geq \omega(1)$.
\end{theorem}
Together with Conjecture~\ref{conj:1}, Theorem~\ref{thm:gap_gen} implies that if we take degree-$|\s{corr}(\calP,\calQ)|^{-1}$ polynomials as a proxy for all efficient algorithms, our calculations predict that even an $n^{O(|\s{corr}(\calP,\calQ)|^{-1})}$ algorithm does not exist when $|\s{corr}(\calP,\calQ)| = o(1)$, which complement our result in Theorem~\ref{thm:sumTest}.

\section{Proofs}\label{sec:proof}

\subsection{Proof of Theorem~\ref{th:2LB}}

\subsubsection{Preliminaries on permutations}

Let $\pi\in\mathbb{S}_n$ be a permutation on $[n]$. For each element $a\in[n]$, its orbit is a cycle $(a_0,\ldots,a_{k-1})$, for some $k\leq n$, where $a_i=\pi^i_a$, $i=0,\ldots,k-1$, and $\pi_{a_{k-1}}=a$. Each permutation can be decomposed as disjoint orbits. Consider the complete graph $\mathbb{K}_n$ with vertex set $[n]$. Each permutation $\pi\in\mathbb{S}_n$ naturally induces
a permutation $\pi^{\s{E}}$ on the edge set of $\mathbb{K}_n$, the set $\binom{{[n]}}{2}$ of all unordered pairs, according to
\begin{align}
    \pi^{\s{E}}((i,j))\triangleq(\pi_i,\pi_j).
\end{align}
We refer to $\pi$ and $\pi^{\s{E}}$ as node permutation and edge permutation, whose orbits are refereed to as node orbits and edge orbits, respectively. For each edge $(i,j)$, let $O_{ij}$ denotes its orbit under $\pi^{\s{E}}$. The cycle structure of the edge permutation is determined by that of the node permutation. Finally, we let $\s{n}_k$ and $\s{N}_k$ denote the number of $k$-node and $k$-edge orbits in $\pi$ and $\pi^{\s{E}}$, respectively.   

\subsubsection{A formula for the second moment}\label{sec:proofLowerBounds}

Using a well-known equivalent characterization of the Bayesian risk function by the total variation distance and the Cauchy-Schwartz inequality, one shows that
\begin{equation}
    \label{eq:LikeBound}
    {\s{R}}^\star=1-d_{\s{TV}}(\P_{\calH_0},\P_{\calH_1})\geq 1-\frac{1}{2}\sqrt{\E_{{\calH_0}}{\pp{\s{L}_n^2}}-1},
\end{equation}
where $\s{L}_n\triangleq\frac{\P_{\calH_1}}{\P_{\calH_0}}$ is the likelihood ratio, and the expectation is taken with respect to $\P_{\calH_0}$. Thus, it is sufficient to find the conditions under which $\E_{{\calH_0}}\pp{\s{L}_n^2}\leq 1+o(1)$. The first step in the calculation calls for the use of the Ingster-Suslina method, stating that by Fubini's theorem, $\E_{{\calH_0}}[\s{L}_n^2]$ may be equivalently written as
\begin{equation}
    \label{eq:LikeSquare}\E_{\calH_0}[\s{L}_n^2]=\E_{\pi\indep\pi'}\pp{\E_{\calH_0}\pp{\frac{\P_{\calH_1|\pi}}{\P_{\calH_0}}\cdot \frac{\P_{\calH_1|\pi'}}{\P_{\calH_0}}}},
\end{equation}
where the expectation is taken with respect to the independent coupling of $\pi$ and $\pi'$, two copies of the uniform measure on $\S_n$. Given $\pi,\pi'$, we note that
\begin{align}
    \frac{\P_{\calH_1|\pi}}{\P_{\calH_0}}\cdot \frac{\P_{\calH_1|\pi'}}{\P_{\calH_0}} = \prod_{i<j} \calL(\s{A}_{ij},\s{B}_{\pi_i\pi_j})\calL(\s{A}_{ij},\s{B}_{\pi'_i\pi'_j})\triangleq \prod_{i<j}\s{X}_{ij},
\end{align}
{where $\calL(x,y) = \frac{\calP(x,y)}{\calQ(x,y)}$, for $(x,y)\in\calX\times\calY$, i.e., the Radon-Nikodym derivative between $\calP$ and $\calQ$.} Thus,
\begin{align}
    \E_{\calH_0}[\s{L}_n^2]=\E_{\pi\indep\pi'}\pp{\E_{\calH_0}\p{\prod_{i<j}\s{X}_{ij}}}.
\end{align}
Fixing $\pi$ and $\pi'$, we first compute the inner expectation. Observe that $\s{X}_{ij}$ may not be independent across different pairs of $(i,j)$. In order to decompose $\prod_{i<j}\s{X}_{ij}$ as a product over independent randomness, we use the notion of cycle decomposition. Define
\begin{align}
    \sigma \triangleq\pi^{-1}\circ\pi',
\end{align}
which is also uniformly distributed on $\mathbb{S}_n$. Let $\sigma^{\s{E}}$ denote the edge permutation induced by $\sigma$. For each edge orbit $O$ of $\sigma^{\s{E}}$, define
\begin{align}
    \s{X}_O\triangleq\prod_{(i,j)\in O}\s{X}_{ij}.
\end{align}
Next, we note that $X_O$ is a function of $(\s{A}_{ij},\s{B}_{\pi_i,\pi_j})_{(i,j)\in O}$. Accordingly, we let $\calO$ denote the partition of edge orbits of $\sigma^{\s{E}}$. Since edge orbits are disjoint, we have
\begin{align}
   \prod_{i<j}\s{X}_{ij} = \prod_{O\in\calO}\s{X}_{O}.
\end{align}
Since $\{\s{A}_{ij}\}_{i<j}$ and $\{\s{B}_{ij}\}_{i<j}$ are independent under $\calH_0$, we conclude $\{\s{X_O}\}_{O\in\calO}$ are statistically independent under $\calH_0$. Therefore, 
\begin{align}
    \E_{\calH_0}[\s{L}_n^2]=\E_{\pi\indep\pi'}\pp{ \prod_{O\in\calO}\E_{\calH_0}(\s{X}_{O})}.\label{eqn:beforeInnerExpec}
\end{align}
Next, we find an expression for $\E_{\calH_0}(\s{X}_{O})$. To that end, we need a few definitions and notations. The kernel $\calL$ defines an operator as follows: for any square-integrable function $f$ under $\calQ$,
\begin{align}
    (\calL f)(x) \triangleq \bE_{Y\sim\calQ}\pp{\calL(x,Y)f(Y)},
\end{align}
In addition, $\calL^2= \calL\circ\calL$ is given by $\calL^2(x,y) = \bE_{Z\sim\calQ}[\calL(x,Z)\calL(Z,y)]$, and $\calL^k$ is similarly defined. Assume that $\calL(x,y) = \calL(y,x)$, and hence $\calL$ is self-adjoint. Furthermore, if we assume that $ \int\int \calL^2(x,y)\calQ(\mathrm{d}x)\calQ(\mathrm{d}y)<\infty$, then $\calL$ is Hilbert-Schmidt. Thus $\calL$ is diagonazable
with eigenvalues $\lambda_i$'s and the trace of $\calL$ is given by $\s{trace}(\calL) = \bE_{Y\sim\calQ}[\calL(Y,Y)] = \sum_{i\in\mathbb{N}}\lambda_i$. We have the following result.
\begin{lemma}\label{lem:secoLOrbit}
    Fixing $\pi$ and $\pi'$, for any orbit O of $\sigma \triangleq\pi^{-1}\circ\pi'$, we have
    \begin{align}
        \E_{\calH_0}(\s{X}_{O}) = \sum_{i\in\mathbb{N}}\lambda_i^{2|O|}. 
    \end{align}
\end{lemma}
\begin{proof}[Proof of Lemma~\ref{lem:secoLOrbit}]
Since $O$ is an edge orbit of $\sigma$, we have $\{\s{B}_{\pi_i\pi_j}\}_{{(i,j)}\in O} = \{\s{B}_{\pi'_i\pi'_j}\}_{{(i,j)}\in O}$ and $(\pi'_i,\pi'_j) = (\pi_{\sigma_i},\pi_{\sigma_j})$. Then,
\begin{align}
   \E_{\calH_0}[\s{X}_{O}] &= \E_{{{A},{B}\sim\calQ}}\pp{\prod_{(i,j)\in O} \calL(\s{A}_{ij},\s{B}_{\pi_i\pi_j})\calL(\s{A}_{ij},\s{B}_{\pi'_i\pi'_j})}\\
   & = \bE\pp{\prod_{(i,j)\in O} \calL^2(\s{B}_{\pi_i\pi_j},\s{B}_{\pi'_i\pi'_j})}\\
   & = \s{trace}\p{\calL^{2|O|}}\\
   & = \sum_{i\in\mathbb{N}}\lambda_i^{2|O|}.
\end{align}
\end{proof}
Thus, applying Lemma~\ref{lem:secoLOrbit} on \eqref{eqn:beforeInnerExpec}, we obtain,
\begin{align}
    \E_{\calH_0}[\s{L}_n^2]&=\E_{\pi\indep\pi'}\pp{ \prod_{O\in\calO}\sum_{i\in\mathbb{N}}\lambda_i^{2|O|}}.\label{eqn:sumprod0}
\end{align}
For a fixed permutation $\pi\in \S_n$ and $k\in {\mathbb{N}}$, let $N_k(\sigma)$ denote the number of {total} $k$-node cycles of $\pi$. Then,
\begin{align}
    \E_{\calH_0}[\s{L}_n^2]&=\E_{\pi\indep\pi'}\pp{ \prod_{k=1}^{\binom{n}{2}}\p{\sum_{i\in\mathbb{N}}\lambda_i^{2k}}^{N_k}}.\label{eqn:sumprod}
\end{align}
The goal is now to understand the behaviour of the expectation w.r.t. the uniform distribution, as $n\to\infty$. Before we continue let us prove the following result.
\begin{lemma}\label{lem:eigen}
    The largest eigenvalue of $\calL$ is 1.
\end{lemma}
\begin{proof}[Proof of Lemma~\ref{lem:eigen}]
We first prove the result for the case where $\calP$ and {$\calQ$} are probability mass functions, i.e., the datasets $\s{X}$ and $\s{Y}$ are discrete and take values over $\calX$ and $\calY$, respectively. The proof for the continuous case, where $\calP$ and {$\calQ$} are probability density functions, is similar. Note that for any square-integrable function $f$ under $\calQ$,
\begin{align}
    (\calL f)(x) &= \bE_{Y\sim \calQ}\pp{\calL(x,Y)f(Y)}\\
    & = \sum_{y\in\calY}\frac{\calP(x,y)}{\calQ(x)\calQ(y)}f(y)Q(y)\\
    & = \frac{1}{\calQ(x)}\sum_{y\in\calY}f(y)\calP(x,y).
\end{align}
Thus, the eigenvalues of $\calL$ are given by the eigenvalues of the following $|\calX|\times|\calY|$ row-stochastic matrix
$\mathbf{M}(x,y)\triangleq \frac{\calP(x,y)}{\calQ(x)}$. Since $\mathbf{M}$ is a row-stochastic matrix its largest eigenvalue is $1$. Indeed, let $\mathbf{1}$ be the all one vector. Then, note that $\mathbf{M}\mathbf{1} = \mathbf{1}$, and thus $1$ is an eigenvalue of $\mathbf{M}$. To show that the this is the largest eigenvalue we use Gershgorin circle theorem. Specifically, take row $\ell$ in $\mathbf{M}$. The diagonal element will be $\mathbf{M}_{\ell\ell}$ and the radius will be $\sum_{i\neq \ell} |\mathbf{M}_{\ell i}| = \sum_{i \neq \ell} \mathbf{M}_{\ell i}$ since all $\mathbf{M}_{\ell i} \geq 0$. This will be a circle with its center in $\mathbf{M}_{\ell \ell} \in [0,1]$, and a radius of $\sum_{i \neq \ell} \mathbf{M}_{\ell i} = 1-\mathbf{M}_{\ell \ell}$. So, this circle will have $1$ on its perimeter. This is true for all Gershgorin circles for this matrix (since $\ell$ was chosen arbitrarily). Thus, since all eigenvalues lie in the union of the Gershgorin circles, all eigenvalues $\lambda_i$ satisfy $|\lambda_i| \leq 1$. 

In the continuous case, we use similar arguments. Specifically, let $\lambda$ be an eigenvalue of $\calL$ with corresponding eigenfunction $f_\lambda$. Let $x^\star = \arg\max_{x\in\calX}|f_\lambda(x)|$. Because $f_\lambda$ is an eigenfunction, we have
\begin{align}
    (\calL f_\lambda)(x) = \bE_{Y\sim \calQ}\pp{\calL(x,Y)f_\lambda(Y)} = \lambda\cdot f_\lambda(x),\label{eqn:GershCont}
\end{align}
for any $x\in\calX$. Substituting $x = x^\star $ in \eqref{eqn:GershCont} and isolating $\lambda$, we get
\begin{align}
    |\lambda| &= \frac{\abs{\bE_{Y\sim \calQ}\pp{\calL({x^\star},Y)f_\lambda(Y)}}}{|f_\lambda(x^\star)|}\\
    &\leq \bE_{Y\sim \calQ}\pp{\calL({x^\star},Y)\abs{\frac{f_\lambda(Y)}{f_\lambda(x^\star)}}}\\
    &\leq \bE_{Y\sim \calQ}\pp{\calL({x^\star},Y)} = 1,
\end{align}
where the first inequality follows from the triangle inequality, the second inequality is because $\abs{\frac{f_\lambda(x)}{f_\lambda(x^\star)}}\leq1$, for any $x\in\calX$, and the last equality follows by the definition of $\calL$. Since $\lambda$ is arbitrary, the above arguments hold for any eigenvalue of $\calL$. Finally, note that the identity eigenfunction, i.e., $f(x) = 1$, for any $x\in\calX$, corresponds to $\lambda=1$.

\end{proof}

\subsubsection{Truncation technique}

We start by describing the general program of conditional second moment method. Note that sometimes certain rare events under $\pr_{\calH_1}$ can cause the second moment to explode, whereas $d_{\s{TV}}(\pr_{\calH_0},\pr_{\calH_1})$ remains bounded away from one. To circumvent such catastrophic events, we can compute the second moment conditioned on events that are typical under $\pr_{\calH_1}$. More precisely, given an event $\calE$ such that $\pr_{\calH_1}(\calE)=1+o(1)$, define the planted model conditional on $\calE$,
\begin{align}
    \pr_{\calH_1}'(\s{A},\s{B},\pi)&\triangleq\frac{\pr_{\calH_1}(\s{A},\s{B},\pi)\Ind_{\ppp{(\s{A},\s{B},\pi)\in\calE}}}{\pr_{\calH_1}(\calE)}\\
    &= (1+o(1))\pr_{\calH_1}(\s{A},\s{B},\pi)\Ind_{\ppp{(\s{A},\s{B},\pi)\in\calE}}.
\end{align}
Then, the likelihood ratio between the conditioned planted model $\pr_{\calH_1}$ and the null model $\pr_{\calH_0}$ is given by
\begin{align}
    \frac{\pr_{\calH_1}'(\s{A},\s{B})}{\pr_{\calH_0}(\s{A},\s{B})} &= 
    {
    \frac{\int_{\mathbb{S}^n}\pr_{\calH_1}'(\s{A},\s{B},\pi)\mathrm{d}\pi}{\pr_{\calH_0}(\s{A},\s{B})}
    }\\     
    &=(1+o(1))\int_{\mathbb{S}^n}\frac{\pr_{\calH_1\vert\pi}(\s{A},\s{B}\vert\pi)\Ind_{\ppp{(\s{A},\s{B},\pi)\in\calE}}}{\pr_{\calH_0}(\s{A},\s{B})}\mathrm{d}\s{Unif}(\pi)\\
    & = (1+o(1))\bE_{\pi}\pp{\frac{\pr_{\calH_1\vert\pi}(\s{A},\s{B}\vert\pi)}{\pr_{\calH_0}(\s{A},\s{B})}\Ind_{\ppp{(\s{A},\s{B},\pi)\in\calE}}}.
\end{align}
Then, by the same reasoning that led to \eqref{eqn:beforeInnerExpec}, the conditional second moment is given by
\begin{align}
    \bE_{\calH_0}\pp{\p{\frac{\pr_{\calH_1}'(\s{A},\s{B})}{\pr_{\calH_0}(\s{A},\s{B})}}^2} =  (1+o(1))\cdot\E_{\pi\indep\pi'}\pp{\E_{\calH_0}\pp{ \prod_{O\in\calO}\s{X}_{O}\Ind_{\ppp{(\s{A},\s{B},\pi)\in\calE}}\Ind_{\ppp{(\s{A},\s{B},\pi')\in\calE}}}}.
\end{align}
Compared to the unconditional second moment, the extra indicators will be useful for ruling out those rare events causing the second moment to blow up. By the data processing inequality of total variation, we have
\begin{align}
    d_{\s{TV}}(\pr_{\calH_1}'(\s{A},\s{B}),\pr_{\calH_1}(\s{A},\s{B}))&\leq d_{\s{TV}}(\pr_{\calH_1}'(\s{A},\s{B},\pi),\pr_{\calH_1}(\s{A},\s{B},\pi))\\
    & = \pr_{\calH_1}\p{(\s{A},\s{B},\pi)\not\in\calE} = o(1).
\end{align}
Thus, {we} arrive at the following conditions for non-detection:
\begin{align}
    &\bE_{\calH_0}\pp{\p{\frac{\pr_{\calH_1}'(\s{A},\s{B})}{\pr_{\calH_0}(\s{A},\s{B})}}^2} = O(1)\quad\longrightarrow\quad  d_{\s{TV}}(\pr_{\calH_0}(\s{A},\s{B}),\pr_{\calH_1}(\s{A},\s{B}))\leq 1-\Omega(1),\\
    &\bE_{\calH_0}\pp{\p{\frac{\pr_{\calH_1}'(\s{A},\s{B})}{\pr_{\calH_0}(\s{A},\s{B})}}^2} = 1+o(1)\quad\longrightarrow\quad  d_{\s{TV}}(\pr_{\calH_0}(\s{A},\s{B}),\pr_{\calH_1}(\s{A},\s{B}))\leq o(1).
\end{align}
Thus, one of our main goals is to choose $\calE$ wisely. 

Let $\calF$ denotes the set of fixed points of $\sigma = \pi^{-1}\circ\pi'$, so that $|\calF| = n_1$. Let
\begin{align}
    \calO_1\triangleq\binom{\calF}{2},
\end{align}
which is a subset of fixed points of the edge permutation. 
With this definition, we have,
\begin{align}
    \prod_{O\in\calO_1}\s{X}_{O} &= \prod_{i<j\in\calF}\s{X}_{ij}\\
    & = \prod_{i<j\in\calF}\calL(\s{A}_{ij},\s{B}_{\pi_i\pi_j})\calL(\s{A}_{ij},\s{B}_{\pi'_i\pi'_j})\\
    & =\exp\p{\sum_{i<j\in\calF}\log\calL(\s{A}_{ij},\s{B}_{\pi_i\pi_j})+\log\calL(\s{A}_{ij},\s{B}_{\pi'_i\pi'_j})}\\
    &=  \exp\p{2\sum_{i<j\in\calF}\log\calL(\s{A}_{ij},\s{B}_{\pi_i\pi_j})}.
\end{align}
We aim to truncate $\prod_{i<j\in\calF}\s{X}_{ij}$ by conditioning on some event under the planted model when $|\calF|=n_1$ is large. For $2\leq k\leq n$, define
\begin{align}
    \zeta(k)\triangleq \binom{k}{2}\cdot E_{\calP}^{-1}\p{\frac{2\log(2en/k)}{k-1}},
\end{align}
and for each $\calS\subset[n]$, define the event,
\begin{align}
    \calE_{\calS}&\triangleq\left\{(\s{A},\s{B},\pi):\sum_{i<j\in\calS}\log\calL(\s{A}_{ij},\s{B}_{\pi_i\pi_j})\leq \zeta(|\calS|)\right\}.
\end{align} 
We condition on the event,
\begin{align}
    \calE\triangleq\bigcap_{\calS\subset[n]:\;\alpha n\leq|\calS|\leq n}\calE_{\calS},
\end{align}
where $\alpha$ is any fixed real-valued number in the interval $(0,1)$. The following lemma proves that $\calE$ holds with high probability under $\calP$.
\begin{lemma}\label{lem:largedeviation}
    Suppose $\alpha n=\omega(1)$. Then, 
    $$
    \pr_{\calH_1}\pp{(\s{A},\s{B},\pi)\in\calE} = 1-e^{-\Omega(n)}.
    $$    
\end{lemma}
\begin{proof}[Proof of Lemma~\ref{lem:largedeviation}]
Fix an integer $\alpha n\leq k\leq n$, and let $m = \binom{k}{2}$. Fix a subset $\calS\subset[n]$ with $|\calS|=k$. Let $\tau,\alpha\in\mathbb{R}$, and $\delta>0$ be a parameter to be specified later on. Then, using Chernoff's bounds we get,
\begin{align}
    \pr_{\calH_1}\p{\sum_{i<j\in\calS}\log\calL(\s{A}_{ij},\s{B}_{\pi_i\pi_j})<m\tau} &= 1-\pr_{\calH_1}\p{\sum_{i<j\in\calS}\log\calL(\s{A}_{ij},\s{B}_{\pi_i\pi_j})\geq m\tau}\\
    &\geq 1-\exp\p{-mE_{\calP}(\tau)}\\
    &= 1-\delta,
\end{align}
where the last inequality holds provided that $\delta\geq \exp\p{-mE_{\calP}(\tau)}$. Therefore, with probability at least $1-\delta$, we obtain that $\sum_{i<j\in\calS}\log\calL(\s{A}_{ij},\s{B}_{\pi_i\pi_j})<mE_{\calP}^{-1}(m^{-1}\log\delta^{-1})$, under $\calP$. Now, there are $\binom{n}{k}\leq \p{\frac{en}{k}}^k$ different choices of $\calS\subset[n]$ with $|\calS|=k$. Thus, by choosing $\frac{1}{\delta} = \p{\frac{2en}{k}}^k$ and applying the union bound, we get that with probability at least $1-\sum_{k=\alpha n}^n2^{-k} = 1-e^{-\Omega(\alpha n)}$ we have $\sum_{i<j\in\calS}\log\calL(\s{A}_{ij},\s{B}_{\pi_i\pi_j})<mE_{\calP}^{-1}(m^{-1}\log\delta^{-1})$, for all $\calS\subset[n]$ with $|\calS|=k$ and all $\alpha n\leq k\leq n$. 
\end{proof}

Now, we compute the conditional second moment. Using the lemma above we can write that,
\begin{align}
    \bE_{\calH_0}\pp{\p{\frac{\pr_{\calH_1}'(\s{A},\s{B})}{\pr_{\calH_0}(\s{A},\s{B})}}^2} = (1+o(1))\cdot\E_{\pi\indep\pi'}\pp{\E_{\calH_0}\pp{ \prod_{O\in\calO}\s{X}_{O}\Ind_{\ppp{(\s{A},\s{B},\pi)\in\calE}}\Ind_{\ppp{(\s{A},\s{B},\pi')\in\calE}}}}.
\end{align}
Next, we fix $\pi$ and $\pi'$ and separately consider the following two cases.
\paragraph{Case I.} When $n_1< \alpha n$, we bound,
    \begin{align}
        \E_{\calH_0}\pp{ \prod_{O\in\calO}\s{X}_{O}\Ind_{\ppp{(\s{A},\s{B},\pi)\in\calE}}\Ind_{\ppp{(\s{A},\s{B},\pi')\in\calE}}}&\leq \E_{\calH_0}\pp{ \prod_{O\in\calO}\s{X}_{O}}\\
        & \leq \prod_{O\in\calO}\p{\sum_{i\in\mathbb{N}}\lambda_i^{2|O|}}.
    \end{align}
\paragraph{Case II.} When $\alpha n\leq n_1\leq n$, we bound,
\begin{align}
        \E_{\calH_0}\pp{ \prod_{O\in\calO}\s{X}_{O}\Ind_{\ppp{(\s{A},\s{B},\pi)\in\calE}}\Ind_{\ppp{(\s{A},\s{B},\pi')\in\calE}}}&\leq \E_{\calH_0}\p{ \prod_{O\in\calO}\s{X}_{O}\Ind_{\ppp{(\s{A},\s{B},\pi)\in\calE_{\calF}}}}\\
        &=\prod_{O\not\in\calO_1}\bE_{\calH_0}[\s{X}_O]\E_{\calH_0}\pp{\prod_{O\in\calO_1} \s{X}_{O}\Ind_{\ppp{(\s{A},\s{B},\pi)\in\calE_{\calF}}}}\\
        & \leq \prod_{O\not\in\calO_1}\p{\sum_{i\in\mathbb{N}}\lambda_i^{2|O|}}\E_{\calH_0}\pp{\prod_{i<j\in\calF} \s{X}_{ij}\Ind_{\ppp{(\s{A},\s{B},\pi)\in\calE_{\calF}}}},
\end{align}
where the first inequality holds because by definition $\calE\subset\calE_\calF$ when $n_1>\alpha n$, the equality follows because $\s{X}_{O}$ is a function of $(\s{A}_{ij},\s{B}_{\pi_i\pi_j})$ that are independent across different $O\in\calO$, and finally $\Ind_{\ppp{(\s{A},\s{B},\pi)\in\calE_{\calF}}}$ is only a function of $(\s{A}_{ij},\s{B}_{\pi_i\pi_j})$ for $(i,j)\in O$ with $O\in\calO_1$. Now, on the event $\calE_\calF$ we have
\begin{align}
\sum_{i<j\in\calF}\log\calL(\s{A}_{ij},\s{B}_{\pi_i\pi_j})\leq \binom{n_1}{2}E_{\calP}^{-1}\p{\frac{2\log(2\alpha e)}{\alpha n-1}}.
\end{align}
This is because $n_1\geq\alpha n$, and the monotonicity of $E_\calP^{-1}(\cdot)$. Indeed, note that $E_{\calP}(\theta)$ is convex and non-negative. Furthermore, we have that $E_{\calP}(\mu) = 0$, where $\mu = \E_{\calP}\pp{\mathscr{L}}$. So $E_{\calP}(\theta)$ is decreasing on $(-\infty,\mu)$ and increasing on $(\mu,\infty)$. Then, the inverse function is monotonic concave on the domain $[0,\infty)$. Accordingly, let  $\zeta\triangleq\binom{n_1}{2}E_{\calP}^{-1}\p{\frac{2\log(2\alpha e)}{\alpha n-1}}$, $\bar\zeta\triangleq E_{\calP}^{-1}\p{\frac{2\log(2\alpha e)}{\alpha n-1}}$, and denote $e_{\s{A},\s{B}}(\calF)\triangleq\sum_{i<j\in\calF}\log\calL(\s{A}_{ij},\s{B}_{\pi_i\pi_j})$. We have,
\begin{align}
    &\E_{\calH_0}\pp{\prod_{i<j\in\calF} \s{X}_{ij}\Ind_{\ppp{(\s{A},\s{B},\pi)\in\calE_{\calF}}}} \leq \bE_{\calQ}\pp{\exp\p{2e_{\s{A},\s{B}}(\calF)}\Ind_{{\{}e_{\s{A},\s{B}}(\calF)\leq\zeta{\}}}}.
\end{align}
Then, for any $\kappa\in[0,1]$,
\begin{align}
   \bE_{\calQ}\pp{\exp\p{2e_{\s{A},\s{B}}(\calF)}\Ind_{{\{}e_{\s{A},\s{B}}(\calF)\leq\zeta{\}}}}&\leq\bE_{\calQ}\pp{\exp\ppp{2\p{\kappa e_{\s{A},\s{B}}(\calF)+(1-\kappa)\zeta}}}\\
   & = \exp\ppp{2(1-\kappa)\zeta+\binom{n_1}{2}\psi_Q(2\kappa)}\\
   &=\exp\ppp{2\zeta-\binom{n_1}{2}\pp{2\kappa\bar\zeta-\psi_\calQ(2\kappa)}}.
\end{align}
Choosing $\kappa = \frac{1}{2}$ yields,
\begin{align}
    2\zeta-\binom{n_1}{2}\pp{2\kappa\bar\zeta-\psi_\calQ(2\kappa)} =2\zeta-\binom{n_1}{2}\bar\zeta.
\end{align}
Therefore,
\begin{align}
    \E_{\calH_0}\pp{\prod_{i<j\in\calF} \s{X}_{ij}\Ind_{\ppp{(\s{A},\s{B},\pi)\in\calE_{\calF}}}} &\leq \exp\p{2\zeta-\binom{n_1}{2}\bar\zeta}\\
    & = \exp\pp{\binom{n_1}{2}\bar\zeta}.
\end{align}
Combining the two cases yields that,
\begin{align}
    \bE_{\calH_0}\pp{\p{\frac{\pr_{\calH_1}'(\s{A},\s{B})}{\pr_{\calH_0}(\s{A},\s{B})}}^2}
    &\leq (1+o(1))\cdot\bE\pp{\prod_{O\in\calO}\p{\sum_{i\in\mathbb{N}}\lambda_i^{2|O|}}\Ind_{{\{}n_1\leq \alpha n{\}}}}\nonumber\\
    &\quad+(1+o(1))\cdot\bE\pp{\prod_{O\not\in\calO_1}\p{\sum_{i\in\mathbb{N}}\lambda_i^{2|O|}}\exp\pp{\binom{n_1}{2}\bar\zeta}\Ind_{{\{}n_1> \alpha n{\}}}}.
\end{align}
Define $\norm{\boldsymbol{\lambda}}_p\triangleq\sum_{i\geq1}\lambda_i^{p}$. Then, since $\sum_kN_k\leq\binom{n}{2}$, we have,
\begin{align}
    \prod_{k=3}^{\binom{n}{2}}\p{\sum_{i\in\mathbb{N}}\lambda_i^{2k}}^{N_k}&\leq \p{\sum_{i\in\mathbb{N}}\lambda_i^{6}}^{\binom{n}{2}}\\
    & = \p{1+\sum_{i\geq1}\lambda_i^{6}}^{\binom{n}{2}}\\
    &\leq \exp\p{n^2\sum_{i\geq1}\lambda_i^{6}}\\
    &= 1+o(1),\label{eqn:vanish3more}
\end{align}
where in the last passage we assumed that  $n^2\norm{\boldsymbol{\lambda}}_6 = o(1)$. Now, note that
\begin{align}
    \prod_{O\not\in\calO_1}\p{\sum_{i\in\mathbb{N}}\lambda_i^{2|O|}} &= \p{\sum_{i\in\mathbb{N}}\lambda_i^{2}}^{n_2}\prod_{k\geq2}\p{\sum_{i\in\mathbb{N}}\lambda_i^{2k}}^{N_k}\\
    & = (1+o(1))\p{\sum_{i\in\mathbb{N}}\lambda_i^{2}}^{n_2}\p{\sum_{i\in\mathbb{N}}\lambda_i^{4}}^{N_2}\\
    &\leq (1+o(1))\exp\p{n_2\cdot\norm{\boldsymbol{\lambda}}_2+N_2\cdot\norm{\boldsymbol{\lambda}}_4}\\
    &\leq (1+o(1))\exp\p{n_2\cdot\norm{\boldsymbol{\lambda}}_2+N_2\cdot\norm{\boldsymbol{\lambda}}_2^2},
\end{align}
where in the second equality we have used \eqref{eqn:vanish3more} and the last inequality follows from the fact that $\norm{\boldsymbol{\lambda}}_4\leq \norm{\boldsymbol{\lambda}}_2^2$. Similarly,
\begin{align}
    \prod_{O\in\calO_1}\p{\sum_{i\in\mathbb{N}}\lambda_i^{2|O|}} = \p{\sum_{i\in\mathbb{N}}\lambda_i^{2}}^{\binom{n_1}{2}}\leq \exp\p{n_1^2\cdot\norm{\boldsymbol{\lambda}}_2/2}.
\end{align}
Hence,
\begin{align}
    \bE_{\calH_0}\pp{\p{\frac{\pr_{\calH_1}'(\s{A},\s{B})}{\pr_{\calH_0}(\s{A},\s{B})}}^2}&\leq (1+o(1))\cdot\bE\pp{\exp\p{(n_1^2/2+n_2)\cdot\norm{\boldsymbol{\lambda}}_2+N_2\cdot\norm{\boldsymbol{\lambda}}_2^2}\Ind_{{\{}n_1\leq \alpha n{\}}}}\nonumber\\
    &+(1+o(1))\cdot\bE\pp{\exp\p{n_2\cdot\norm{\boldsymbol{\lambda}}_2+N_2\cdot\norm{\boldsymbol{\lambda}}_2^2}\exp\pp{\binom{n_1}{2}\bar\zeta}\Ind_{{\{}n_1> \alpha n{\}}}}.
\end{align}
We upper bound the two terms separately. To bound the first term, we apply \cite[eq. (48)]{YihongJiamingSophie} with $\mu = \norm{\boldsymbol{\lambda}}_2/2$, $\nu=0$, $a=0$, $\tau = \norm{\boldsymbol{\lambda}}_2$, and $b=\alpha n$. If,
\begin{align}
    \norm{\boldsymbol{\lambda}}_2^2 = o(n^{-1}),\ \textnormal{and}\ \frac{\alpha\norm{\boldsymbol{\lambda}}_2n}{2}+2-\log(\alpha n)\leq0,\label{eqn:cond1lower}
\end{align}
then \cite[eq. (48)]{YihongJiamingSophie} implies that
\begin{align}
    \bE\pp{\exp\p{(n_1^2/2+n_2)\cdot\norm{\boldsymbol{\lambda}}_2+N_2\cdot\norm{\boldsymbol{\lambda}}_2^2}\Ind_{{\{}n_1\leq \alpha n{\}}}}\leq 1+o(1).
\end{align}
To bound the second term, we apply \cite[eq. (47)]{YihongJiamingSophie} with $\mu = \frac{\bar\zeta}{2}$, $\nu=0$, $a=\alpha n$, $\tau = \norm{\boldsymbol{\lambda}}_2$, and $b=n$. If,
\begin{align}
    \norm{\boldsymbol{\lambda}}_2^2 = o(n^{-1}),\ \textnormal{and}\ \frac{\bar\zeta}{2}n+2-\log n\leq0,\label{eqn:cond2lower}
\end{align}
then \cite[eq. (47)]{YihongJiamingSophie} implies that 
\begin{align}
    \bE\ppp{\exp\p{n_2\cdot\norm{\boldsymbol{\lambda}}_2+N_2\cdot\norm{\boldsymbol{\lambda}}_2^2}\exp\pp{\zeta}\Ind_{{\{}n_1> \alpha n{\}}}}\leq 1+o(1).
\end{align}
Note that the condition in the right-hand-side of \eqref{eqn:cond1lower} can be written as,
\begin{align}
    \norm{\boldsymbol{\lambda}}_2\leq\frac{2\log(\alpha n)-4}{\alpha n},\label{eqn:cond3lower}
\end{align}
which implies the condition in the left-hand-side of \eqref{eqn:cond1lower} and \eqref{eqn:cond2lower}, and thus dominates. This condition also implies the assumption above that $n^2\norm{\boldsymbol{\lambda}}_6 = o(1)$. Note that from the proof of Lemma~\ref{lem:secoLOrbit}, we can see that 
\begin{align}
\norm{\boldsymbol{\lambda}}_2&=\sum_{i\in\mathbb{N}}\lambda_i^2-1=\bE_{\calQ}\pp{\calL^2}-1=\chi^2\p{\calP || \calQ},
\end{align}
so the condition in \eqref{eqn:cond3lower} can be written as, 
\begin{align}
    \chi^2\p{\calP || \calQ}\leq\frac{2\log(\alpha n)-4}{\alpha n}.\label{eqn:cond4lower}
\end{align}
The condition in the right-hand-side of \eqref{eqn:cond2lower} can be written as,
\begin{align}
    \bar\zeta\leq\frac{2\log n-4}{n}.\label{eqn:cond5lower}
\end{align} 
Notice that the right-hand-side of \eqref{eqn:cond4lower} is monotonically decreasing as a function of $\alpha$. 

We would like to simplify the required conditions. 
\begin{lemma}
    The condition $\bar\zeta = E_{\calP}^{-1}\p{\frac{2\log(2\alpha e)}{\alpha n-1}}$ is contained in
    \begin{align}
        \bar\zeta \le d_{\s{KL}}(\calP || \calQ) + O(\lambda^\star \s{Var}_{\calP}(\mathscr{L})) + \frac{2\log(2\alpha e)}{\lambda^\star (\alpha n - 1)},
    \end{align}
    for any $\lambda^\star\to0$ and $\alpha\in(0,1)$.
\end{lemma}
\begin{proof}
    Recall the definition of $E_{\calP}(\theta)$ in \eqref{eq:Legendre_trans_def}, then for any $\lambda^\star\in\mathbb{R}$ it holds that $E_{\calP}(\theta)) \ge \lambda^{\star}\theta  - \psi_{\calP}(\lambda^{\star})$. Specifically, for $\theta = \bar{\zeta}$ we have on the one hand
    \begin{align}
        E_{\calP}(\bar\zeta) \ge \lambda^{\star}\bar\zeta  - \psi_{\calP}(\lambda^{\star}),
    \end{align} 
    and on the other hand $E_{\calP}\p{\bar\zeta} = \frac{2\log(2\alpha e)}{\alpha n-1}$. Therefore,
    \begin{align}
        \bar\zeta \le \frac{\psi_{\calP}(\lambda^\star)}{\lambda^\star} + \frac{2\log(2\alpha e)}{\lambda^{\star}(\alpha n-1)}. \label{eq:condition_tosimplify}
    \end{align}
    Now, let us find the Maclaurin series of $\psi_{\calP}(\lambda)$. Note that
    \begin{align}
        \lim_{\lambda \to 0} \psi_{\calP}(\lambda) &= 0,\\
        \lim_{\lambda \to 0} \frac{\partial\psi_{\calP}({\partial}\lambda)}{{\partial}\lambda} &= d_{\s{KL}}\p{\calP || \calQ},
    \end{align}
    and,
    \begin{align}
        \lim_{\lambda \to 0} \frac{\partial^2\psi_{\calP}({\partial}\lambda)}{{\partial}\lambda^2} =  \s{Var}_{\calP}\p{\mathscr{L}}.
    \end{align}
    Hence,
    \begin{align}
        \psi_{\calP}(\lambda) = \lambda d_{\s{KL}}(\calP || \calQ) + O\p{\frac{\lambda^2}{2}\s{Var}_{\calP}\p{\mathscr{L}}},
    \end{align}
    and \eqref{eq:condition_tosimplify} can be written as follows
    \begin{align}
        \bar\zeta \le d_{\s{KL}}(\calP || \calQ) + O\p{\lambda^\star\s{Var}_{\calP}\p{\mathscr{L}}} + \frac{2\log(2\alpha e)}{\lambda^{\star}(\alpha n-1)}. 
    \end{align}
\end{proof}
Choosing $\bar\zeta = d_{\s{KL}}(\calP || \calQ) + O\p{\lambda^\star\s{Var}_{\calP}\p{\mathscr{L}}} + \frac{2\log(2\alpha e)} {\lambda^{\star}(\alpha n-1)}$, condition \eqref{eqn:cond5lower} gives
\begin{align}
    d_{\s{KL}}(\calP || \calQ) + O\p{\lambda^\star\s{Var}_{\calP}\p{\mathscr{L}}} + \frac{2\log(2\alpha e)} {\lambda^{\star}(\alpha n-1)} \le \frac{(2-\epsilon)\log n}{n}.
\end{align}
We note that for $\omega\p{1/\log n}=\lambda^\star = o\p{1}$ it holds that $\omega(n^{-1})=\frac{2\log(2\alpha e)} {\lambda^{\star}(\alpha n-1)} = o\p{\frac{\log n}{n}}$, and therefore, we get the condition,
\begin{align}
 d_{\s{KL}}(\calP || \calQ) + O\p{\lambda^\star\s{Var}_{\calP}\p{\mathscr{L}}} \le \frac{(2-\epsilon')\log n}{n}.
\end{align}

\subsection{Proof of Theorem~\ref{thm:GLRT}}

Recall the GLRT in \eqref{eqn:testscanmain}, and let us analyze the Type-I error probability. Let $\lambda\geq0$. We have,
\begin{align}
    \calE_1&\triangleq\pr_{\calH_0}\pp{\phi_{\s{GLRT}}(\s{A},\s{B})=1} \\
    &= \pr_{\calH_0}\pp{\max_{\pi\in\mathbb{S}_n}\sum_{1\leq i<j\leq n}\log \calL(\s{A}_{ij},\s{B}_{\pi_i\pi_j})\geq n^{(2)}\cdot\tau_{\s{GLRT}}}\\
    & = \pr_{\calH_0}\pp{\bigcup_{\pi\in\mathbb{S}_n}\sum_{1\leq i<j\leq n}\log \calL(\s{A}_{ij},\s{B}_{\pi_i\pi_j})\geq n^{(2)}\cdot\tau_{\s{GLRT}}}\\
    &\leq n!\cdot\pr_{\calH_0}\pp{\sum_{1\leq i<j\leq n}\log \calL(\s{A}_{ij},\s{B}_{ij})\geq n^{(2)}\cdot\tau_{\s{GLRT}}}\\
    &\leq n!\cdot e^{-n^{(2)}\cdot\lambda\tau_{\s{GLRT}}}\bE_{\calH_0}\pp{\exp\p{\lambda\sum_{1\leq i<j\leq n}\log \calL(\s{A}_{ij},\s{B}_{ij})}}\\
    &= n!\cdot e^{-n^{(2)}\cdot\lambda\tau_{\s{GLRT}}}\pp{\bE_{\calH_0}\pp{\exp\p{\lambda\log\calL(\s{A}_{11},\s{B}_{11})}}}^{n^{(2)}}\\
    & = n!\cdot e^{-n^{(2)}\cdot\lambda\tau_{\s{GLRT}}}\pp{\bE_{\calQ}\pp{\exp\p{\lambda\log\frac{\calP({A},{B})}{\calQ({A},{B})}}}}^{n^{(2)}}\\
    &= \exp\pp{\log n!-n^{(2)}\cdot\lambda\tau_{\s{GLRT}}+n^{(2)}\cdot\psi_{\calQ}(\lambda)},
\end{align}
where the first inequality follows from the union bound and $|\mathbb{S}_n|=n!$, and the second inequality is by Chernoff's bound. For $\tau_{\s{GLRT}}\in(-d_{\s{KL}}(Q||P),d_{\s{KL}}(P||Q))$, we may take $\lambda\geq0$ so that $\lambda\tau_{\s{GLRT}}-\psi_{\calQ}(\lambda)$ is arbitrarily close to $E_{\calQ}(\tau_{\s{GLRT}})$. This implies that
\begin{align}
  \calE_1\leq \exp\pp{n\log \frac{n}{e}+\log n+1-n^{(2)}\cdot E_{\calQ}(\tau_{\s{GLRT}})},  
\end{align}
where we have used Stirling approximation that $n!\leq en^{n+1}e^{-n}$. Therefore, we see that if,
\begin{align}
    E_{\calQ}(\tau_{\s{GLRT}})\geq \frac{2\log (n/e)}{n-1}+\frac{2(1+\log n)}{n(n-1)}+\omega(n^{-2}),\label{eqn:condScan}
\end{align}
then the Type-I error probability goes to zero. 
Next, we bound the Type-II error probability as follows. Under $\calH_1$, since our proposed test is invariant to reordering of $\s{A}$ and $\s{B}$, we may assume without loss of generality that the latent permutation is the identity one, i.e., $\sigma=\s{Id}$. Then, for $\lambda\leq0$,
\begin{align}
    \calE_2&\triangleq\pr_{\calH_1}\pp{\phi_{\s{GLRT}}(\s{A},\s{B})=0}\\
    &= \pr_{\calH_1}\pp{\max_{\pi\in\mathbb{S}_n}\sum_{1\leq i<j\leq n}\log \calL(\s{A}_{ij},\s{B}_{\pi_i\pi_j})< n^{(2)}\cdot\tau_{\s{GLRT}}}\\
    & = \pr_{\calH_1}\pp{\sum_{1\leq i<j\leq n}\log \calL(\s{A}_{ij},\s{B}_{ij})< n^{(2)}\cdot\tau_{\s{GLRT}}}\\
    &\leq e^{-n^{(2)}\cdot\lambda\tau_{\s{GLRT}}}\bE_{\calH_1}\pp{\exp\p{\lambda\sum_{1\leq i<j\leq n}\log \calL(\s{A}_{ij},\s{B}_{ij})}}\\
    & = e^{-n^{(2)}\cdot\lambda\tau_{\s{GLRT}}}\pp{\bE_{\calP}\pp{\exp\p{\lambda\log\frac{\calP({A},{B})}{\calQ({A},{B})}}}}^{n^{(2)}}\\
    & = \exp\pp{n^{(2)}\cdot\psi_{\calP}(\lambda)-n^{(2)}\cdot\lambda\tau_{\s{GLRT}}}.
\end{align}
Again, because $\tau_{\s{GLRT}}\in(-d_{\s{KL}}(\calQ||\calP),d_{\s{KL}}(\calP||\calQ))$, we may take $\lambda\leq0$ so that $\lambda\tau_{\s{GLRT}}-\psi_{\calP}(\lambda)$ is arbitrarily close to $E_{\calP}(\tau_{\s{GLRT}})$. Therefore,
\begin{align}
    \calE_2\leq\exp\p{-n^{(2)}\cdot E_{\calP}(\tau_{\s{GLRT}})},
\end{align}
and thus if,
\begin{align}
    E_{\calP}(\tau_{\s{GLRT}})=\omega(n^{-2}),
\end{align}
then the Type-II error probability goes to zero. 

\subsection{Proof of Theorem~\ref{thm:sumTest}}
We assume { $\s{corr}(\calP,\calQ)\triangleq\frac{\s{cov}_{\calP}({A},{B})}{\s{Var}_{\calQ}(;{A})}>0$}. The complementary case follows from the same arguments. For simplicity of notation we define $\calT(\s{A},\s{B})\triangleq\sum_{(i,j)\in\Lambda}(\s{A}_{ij}-\s{B}_{ij})${, where $\Lambda\triangleq\{(i,j):1\leq i<j\leq n\}$}. Then, note that $\bE_{\calH_0}[\calT(\s{A},\s{B})] = 0$, and $\s{Var}_{\calH_0}(\calT(\s{A},\s{B})) = n(n-1)\cdot\s{Var}_{\calQ}({A})$. Similarly, $\bE_{\calH_1}[\calT(\s{A},\s{B})] = 0$, and $\s{Var}_{\calH_1}[\calT(\s{A},\s{B})] = n(n-1)\cdot\s{Var}_{\calQ}({A})\cdot(1-\s{corr}(\calP,\calQ))$. Next, we use Berry-Esseen theorem to approximate $\calT(\s{A},\s{B})$ by a Gaussian random variables, under $\calH_0$ and $\calH_1$. Let $G_0 \sim \calN(0,n(n-1)\cdot\s{Var}_{\calQ}({A}))$ and $G_1\sim\calN(0,n(n-1)\cdot\s{Var}_{\calQ}({A})\cdot(1-\s{corr}(\calP,\calQ)))$. Then, using \cite[Theorem 5.5]{Valentin}, we have
    \begin{align}
        \alpha\triangleq\pr_{\calH_0}{\p{|\calT(\s{A},\s{B})|\geq\theta}} &\geq \pr(|G_0|\geq\theta)-O\p{\frac{\bE_{\calQ}{\pp{|{A}-{B}|^3}}}{n\cdot\s{Var}^{3/2}_{\calQ}({A})}}.\label{eqn:alphaBound}
    \end{align}
    Similarly,
    \begin{align}
        \beta\triangleq&\pr_{\calH_1}{\p{|\calT(\s{A},\s{B})|\geq\theta}}\leq \pr(|G_1|\geq\theta) +O\p{\frac{\bE_{\calP}{\pp{|{A}-{B}|^3}}}{n\cdot\s{Var}^{3/2}_{\calQ}({A})(1-\s{corr}(\calP,\calQ))^{3/2}}}.\label{eqn:betaBound}
    \end{align}
    Now note that,
    \begin{align}
        &\pr(|G_0|\geq\theta)-\pr(|G_1|\geq\theta) \geq d_{\s{TV}}\p{\calN(0,1)||\calN(0,1-\s{corr}(\calP,\calQ))}= \Omega(1),
    \end{align}
    where the inequality is because of the definition of the total variation distance, and the last equality is because $\s{corr}(\calP,\calQ)=\Omega(1)$. Furthermore, under the Theorem statement, both terms at the r.h.s. of \eqref{eqn:alphaBound} and \eqref{eqn:betaBound} are $o(1)$. Thus, noticing that $\s{R}_n(\phi_{\s{sum}}) = 1-(\alpha-\beta)$, and combining the last two results we obtain that $\s{R}(\phi_{\s{sum}})<1$, as $n\to\infty$.

\subsection{Proof of Theorem~\ref{thm:gap_gen}}
Recall 
the sequence of orthogonal univariate polynomials with respect to the marginal distribution of $\calP_{AB}$,
$\{\mathscr{P}_\ell\}_{\ell\geq0}$, such that $\mathscr{P}_{\ell}$ is a polynomial of order $\ell$. 
The multivariate polynomials in $n\times n$ variables are indexed by $\theta \in \mathbb{N}^{n\times n}$, and are merely products of the univariate polynomials, i.e., $F_{\theta}(\s{A})\triangleq\prod_{0< i<j\le n} \mathscr{P}_{\theta_{i,j}}(\s{A}_{i,j})$. For some fixed value  $\s{D}$, the norm of ${\s{L}_n}_{,\le \s{D}}$ can be decomposed into its orthogonal components up to degree $\s{D}$ using the basis formed by the family of the multivariate polynomials. That is,
\begin{align}
    \norm{\s{L}_{n}(\s{A},\s{B})_{\le \s{D}}}^2 
    &= \sum_{d=0}^{\s{D}}\sum_{\substack{\alpha,\beta \in \mathbb{N}^{n\times n}\\ \abs{\alpha}+\abs{\beta}=d}} \inprod{\s{L}_n(\s{A}, \s{B}), F_{\alpha,\beta}(\s{A},\s{B})}^2_{\calH_0}. \label{eq:LR_D_measure}
\end{align}
Let us focus on the projection coefficients in \eqref{eq:LR_D_measure}.
\begin{align}
    \inprod{\s{L}_n(\s{A}, \s{B}), F_{\alpha,\beta}(\s{A},\s{B})}_{\calH_0} 
    &= \E_{\calH_0}\pp{\s{L}_{n}(\s{A},\s{B})F_{\alpha,\beta}(\s{A},\s{B})} \\
    &= \E_{\calH_1}\pp{F_{\alpha,\beta}(\s{A},\s{B})}\\
    &= \E_{\pi}\pp{\E_{\calH_1\vert\pi}\pp{F_{\alpha,\beta}(\s{A},\s{B})}}\\
    &= \E_{\pi}\pp{\E_{\calH_1\vert\pi}\pp{\prod_{0<i<j\le n}\mathscr{P}_{\alpha_{ij}}(\s{A}_{ij})\mathscr{P}_{\beta_{\pi_i\pi_j}}(\s{B}_{\pi_i\pi_j})}}\\
    &= \E_{\pi}\pp{\prod_{0<i<j\le n}\E_{\calH_1\vert\pi}\pp{\mathscr{P}_{\alpha_{ij}}(\s{A}_{ij})\mathscr{P}_{\beta_{\pi_i\pi_j}}(\s{B}_{\pi_i\pi_j})}}, \label{eq:general_case_decomposition}
\end{align}
where the last equality follows from the independence of the pairs $(\s{A}_{ij}, \s{B}_{\pi_i\pi_j})_{i<j}$ conditioned on the hidden permutation $\pi$.

In order to simplify \eqref{eq:general_case_decomposition} for general distributions, further assumptions on $(\calP,\calQ)$ are needed. Thus, we assume the class of distributions $\calD_3$ 
such that for each $\ell\in\mathbb{N}$, the projection $\mathbb{E}_{\calP}[{B}^{\ell}\vert {A}=a]$ is a polynomial of degree $\ell$ in $a$ with a leading coefficient $\kappa_{\ell}\le \abs{\s{corr}(\calP,\calQ)}^{\ell}$. 

\begin{lemma}\label{lem:class3_prop}
For the class of distributions $\calD_3$, it holds that
\begin{align}
    \mathbb{E}_{\calP}[\mathscr{P}_n({B})\vert {A}] = \kappa_n\mathscr{P}_n({A}),
\end{align}
where $\kappa_n$ is the leading coefficient of $\mathbb{E}_{\calP}[{B}^{n}\vert {A}=a]$. 
\end{lemma}
Before we get into the proof of Lemma \ref{lem:class3_prop}, we first utilize it to conclude Theorem~\ref{thm:gap_gen}. 
Let ${A}$ be drawn from $\calP_{{A}}$, then for the class of $\calD_3$, Lemma \ref{lem:class3_prop} states that
\begin{align}
    \E_{\calH_1\vert\pi}\pp{\mathscr{P}_{\alpha_{ij}}(\s{A}_{ij})\mathscr{P}_{\beta_{\pi_i\pi_j}}(\s{B}_{\pi_i\pi_j})} 
    &= \E_{{A}\sim\calP_{{A}}}\pp{\mathscr{P}_{\alpha_{ij}}({A})
    \kappa_{\beta_{\pi_i\pi_j}} \mathscr{P}_{\beta_{\pi_i\pi_j}}({A})}\\
    &= \kappa_{\beta_{\pi_i\pi_j}} \delta(\alpha_{i,j} - \beta_{\pi_i,\pi_j}),
\end{align}
where $\kappa_{\beta_{\pi_i\pi_j}}$ are the leading coefficients of the polynomial expansions of the conditional expectations $\mathbb{E}_{\calP}[B^{\beta_{\pi_i\pi_j}}\vert A=a]$. 
Thus, we get
\begin{align}
    \inprod{\s{L}_n(\s{A}, \s{B}), F_{\alpha,\beta}(\s{A},\s{B})}_{\calH_0} 
    &= \E_{\pi}\pp{\prod_{0<i<j\le n}\E_{\calH_1\vert\pi}\pp{\mathscr{P}_{\alpha_{ij}}(\s{A}_{ij})\mathscr{P}_{\beta_{\pi_i\pi_j}}(\s{B}_{\pi_i\pi_j})}} \\
    &= \E_{\pi}\pp{\prod_{0<i<j\le n}\kappa_{\beta_{\pi_i\pi_j}}\delta(\alpha_{ij}-\beta_{\pi_i\pi_j})}.
\end{align}

Substituting this back into~\eqref{eq:LR_D_measure}, we get 
\begin{align}
    \norm{\s{L}_{n}(\s{A},\s{B})_{\le \s{D}}}^2 
    &= \sum_{d=0}^{\s{D}}
    \sum_{\substack{\alpha,\beta \in \mathbb{N}^{n\times n}, \\ \abs{\alpha}+\abs{\beta}=d}} \p{\E_{\pi}\pp{\prod_{0<i<j\le n}\kappa_{\beta_{\pi_i\pi_j}}\delta(\alpha_{ij}-\beta_{\pi_i\pi_j})}}^2. \label{eq:norm_LDP}
\end{align}

The coefficients $\kappa_{\beta_{\pi_i\pi_j}}$ play a crucial role in quantifying the strength of the relationship between ${A}$ and ${B}$, as they represent the leading factors in the polynomial expansions of the conditional expectation $\E_{\calP}\pp{{B}^{\beta_{\pi_i\pi_j}}\vert {A}=a}$. 
For distributions in $\calD_3$, 
we have that $\kappa_{\beta_{\pi_i\pi_j}} \le \abs{\s{corr}(\calP,\calQ)}^{\beta_{\pi_i\pi_j}}$. Therefore, we get
\begin{align}
 \E_{\pi}\pp{\prod_{0<i<j\le n}\kappa_{\beta_{\pi_i\pi_j}}\delta(\alpha_{ij}-\beta_{\pi_i\pi_j})} &= \prod_{0<i<j\le n}\kappa_{\alpha_{ij}}\cdot\E_{\pi}\pp{\prod_{0<i<j\le n}\delta(\alpha_{ij}-\beta_{\pi_i\pi_j})}\\
 &=\prod_{0<i<j\le n}\kappa_{\alpha_{ij}}\cdot\E_{\pi}\pp{\Ind_{\ppp{\pi(\alpha)=\beta}}}\\
&\le\abs{\s{corr(\calP,\calQ)}}^{\abs{\alpha}}\P\p{\pi(\alpha)=\beta}.
\label{eq:corr_LDP}
\end{align}
The permutation $\pi$ is a permutation on the vertex set $[n]$. Thus, $\pi(\alpha)=\beta$ means that the matrix $\beta\in \mathbb{N}^{n\times n}$ is the adjacency matrix of a weighted graph that is isomorphic to the weighted graph that is represented by the adjacency matrix $\alpha\in \mathbb{N}^{n\times n}$ . For $\pi\sim\s{Unif}(\mathbb{S}_n)$, we have
\begin{align}
\P\p{\pi(\alpha)=\beta}&=\frac{1}{\abs{\ppp{\beta\in \mathbb{N}^{n\times n}:\beta\cong \alpha}}}\Ind_{\ppp{\beta\cong \alpha}}.\label{eq:measure_LDP}
\end{align}
Now, combining \eqref{eq:norm_LDP}, \eqref{eq:corr_LDP} and \eqref{eq:measure_LDP} we get

\begin{align}
    \norm{\s{L}_{n}(\s{A},\s{B})_{\le \s{D}}}^2 
    &\le \sum_{d=0}^{\s{D}}
    \sum_{\substack{\alpha,\beta \in \mathbb{N}^{n\times n}, \\ \abs{\alpha}+\abs{\beta}=d}} \s{corr}(\calP,\calQ)^{2\abs{\alpha}}\frac{1}{\abs{\ppp{\beta\in \mathbb{N}^{n\times n}:\beta\cong \alpha}}^2}\Ind_{\ppp{\beta\cong \alpha}}\\&=\sum_{d=0}^{\s{D}}\sum_{\calG:\abs{\s{G}}=d/2}\s{corr}(\calP,\calQ)^{2\abs{\s{G}}}\frac{1}{\abs{\ppp{\beta\in \mathbb{N}^{n\times n}:\beta\cong \s{G}}}^2}\nonumber\\
    &\qquad\qquad\qquad\qquad\qquad\qquad\cdot\abs{\ppp{\alpha,\beta\in \mathbb{N}^{n\times n}:\alpha\cong\beta\cong\s{G}}}\\
    &=\sum_{d=0}^{\s{D}}\sum_{\calG:\abs{\s{G}}=d/2}\s{corr}(\calP,\calQ)^{2\abs{\s{G}}}.
  \end{align}  
 The sum over the set of different isomorphism classes of G is upper bounded by a sum over $k$ and $\ell$ of unweighted graphs with $k$ edges and $\ell$ vertices, multiplied by the number of options to assign weights to the edges such that the total weight of the graph will be $d/2$. The number of options to assign the weights is the number of compositions of $d/2$ into $k$ parts, which is ${{d/2-1}\choose{k-1}}$. We therefore have,

\begin{align}
\sum_{d=0}^{\s{D}}\sum_{\calG:\abs{\s{G}}=d/2}\s{corr(\calQ,\calP)^{2\abs{\s{G}}}}&\le\sum_{d=0}^{\s{D}}\s{corr}(\calP,\calQ)^{d}\sum_{k=0}^{d/2}{{d/2-1}\choose{k-1}}\sum_{\ell=0}^{2k}{{\ell\p{\ell-1}/2}\choose{k}}\label{eq:bound_LDP}\\&\le\sum_{d=0}^{\s{D}}\s{corr}(\calP,\calQ)^{d}\sum_{k=0}^{d/2}{{d/2-1}\choose{k-1}}{{2k^2}\choose{k+1}}\\&\le\sum_{d=0}^{\s{D}}\s{corr}(\calP,\calQ)^{d}\sum_{k=0}^{d/2}\frac{\p{d/2-1}^{k-1}}{(k-1)!}\cdot\frac{\p{2k}^{2k}}{(k-1)!}\\&\le\sum_{d=0}^{\s{D}}\sum_{k=0}^{d/2}\s{corr}(\calP,\calQ)^{2k}\cdot\frac{\p{d}^{k-1}}{(k-1)!}\cdot\frac{\p{d}^{2k}}{(k-1)!}\\&=\sum_{d=0}^{\s{D}}\sum_{k=0}^{d/2}\s{corr}(\calP,\calQ)\cdot\frac{\p{\s{corr}(\calP,\calQ)\cdot d}^{k-1}}{(k-1)!}\cdot\frac{\p{\s{corr}(\calP,\calQ)\cdot d^2}^k}{(k-1)!}\\&\le\sum_{d=0}^{\s{D}}\s{corr}(\calP,\calQ)\cdot d\cdot e^{\s{corr}(\calP,\calQ)\cdot d}\cdot \s{corr}(\calP,\calQ)\cdot d^2\cdot e^{\s{corr}(\calP,\calQ)\cdot d^2}\\&\le\p{\s{corr}(\calP,\calQ)\cdot \s{D}^2}^2\cdot e^{\s{corr}(\calP,\calQ)\cdot\s{D}}\cdot e^{\s{corr}(\calP,\calQ)\cdot \s{D}^2}.
\end{align}
Thus, 
  \begin{align}  
  \norm{\s{L}_{n}(\s{A},\s{B})_{\le \s{D}}}^2 
    &\le\p{\s{corr}(\calP,\calQ)\cdot \s{D}^2}^2\cdot e^{\s{corr}(\calP,\calQ)\cdot \s{D}}\cdot e^{\s{corr}(\calP,\calQ)\cdot \s{D}^2}.
\end{align}
We can see that $\norm{\s{L}_{n}(\s{A},\s{B})_{\le \s{D}}}^2 \leq O(1)$, for any $\s{D}=O(|\s{corr}(\calP,\calQ)|^{-1/2})$. It is now left to prove Lemma \ref{lem:class3_prop}.

\begin{proof}[Proof of Lemma \ref{lem:class3_prop}]

We start with a few simple observations. Let $\zeta_\ell\triangleq\bE\pp{\calP_\ell(A)\calP_\ell(B)}$. Then, since we assume that $\calP,\calQ\in\calD_3$, namely, for each $\ell\in\mathbb{N}$, the projection $\mathbb{E}_{\calP}[B^{\ell}\vert A=a]$ is a polynomial of degree $\ell$ in $a$, we have for $\ell\neq m$
\begin{align}
    \bE\pp{\calP_\ell(A)\calP_m(B)} &= \bE\pp{\calP_\ell(A)\bE\p{\calP_m(B)\vert A}}\\
    & = \bE\pp{\calP_\ell(A)\s{pol}_{m}(A)}\\
    &=0,
\end{align}
where the second equality is because $\calP,\calQ\in\calD_3$ and $\s{pol}_m(A)$ designates a polynomial of order $m$ in $A$, and the last equality follows from the fact that $\{\calP_m\}$ is a set of orthonormal polynomials. To conclude, 
\begin{align}
    \bE\pp{\calP_\ell(A)\calP_m(B)} = \zeta_n\cdot\delta(m-\ell).
\end{align}
Recall that $\kappa_\ell$ is defined as the leading coefficient of $\mathbb{E}_{\calP}[B^{\ell}\vert A=a]$. We want to prove that 
\begin{align}
    \mathbb{E}_{\calP}[\mathscr{P}_n(B)\vert A] = \kappa_n\mathscr{P}_n(A),\label{eqn:1stLemma5}
\end{align}
and we remind that we assume that for each $\ell\in\mathbb{N}$, the projections $\mathbb{E}_{\calP}[B^{\ell}\vert A=a]$ and $\mathbb{E}_{\calP}[A^{\ell}\vert B=b]$ are polynomials of degree $\ell$ in $a$ and $b$, respectively. This assumption clearly implies that $\mathbb{E}_{\calP}[\mathscr{P}_n(B)\vert A]$ is a polynomial with degree $n$, with probability one. Now, note that
\begin{align}
    \bE\pp{A^\ell\bE\pp{\mathscr{P}_n(B)\vert A}} &=\bE\pp{A^\ell\mathscr{P}_n(B)}\\
    & = \bE\pp{\mathscr{P}_n(B)\bE\pp{\left.A^\ell\right| B}},
\end{align}
for $\ell=0,1,\ldots,n-1$. However, since we assume that 
$\bE\pp{\left.A^\ell\right| B=b}$ is a polynomial of degree $\ell$, then it follows by the orthogonality of $\{\mathscr{P}_n\}$ that,
\begin{align}
    \bE\pp{\mathscr{P}_n(B)\bE\pp{\left.A^\ell\right| B}}=0,
\end{align}
and so $\bE\pp{A^\ell\bE\pp{\mathscr{P}_n(B)\vert A}}=0$ as well, for each $\ell=0,1,\ldots,n-1$. This in turn implies that $\bE\pp{\mathscr{P}_n(B)\vert A}$ must be proportional to $\mathscr{P}_n(A)$, i.e., for some constant $\kappa_n$ it holds that $\mathbb{E}_{\calP}[\mathscr{P}_n(B)\vert A] = \kappa_n\mathscr{P}_n(A)$, as required. Finally, we show that $\zeta_\ell = \kappa_\ell$. Indeed, from \eqref{eqn:1stLemma5}, we have
\begin{align}
    \zeta_\ell &= \bE\pp{\calP_\ell(A)\bE\p{\calP_\ell(B)\vert A}}\\
    & = \kappa_\ell\bE\pp{\calP^2_\ell(A)}\\
    & = \kappa_\ell,
\end{align}
which concludes the proof.

\end{proof}

\section{Conclusion and Outlook}

In this paper, we analyzed the problem of detecting the dependency weighted random graphs. To that end, we derived both lower and upper bounds on the phase transition boundaries at which detection is possible or impossible. We show that these bounds are tight in several special cases considered in the literature, and in other distributions as well. We observed a gap between the statistical limits we derive and the performance of the efficient algorithms we construct.
We conjecture that this gap is in fact inherent, and to provide an evidence for this conjecture we
follow show that the class of low-degree polynomials
fail the conjectureally hard regime. We hope our work has opened more doors than it closes. Let us mention briefly a few interesting questions for future research:
    \begin{enumerate}
        \item It would be interesting to devise universal algorithms which are independent of the underlying generative distributions, which assumed known in our paper.
        \item In this paper, we investigated the average-case/noisy version of the famous worst-case graph isomorphism problem; accordingly it would be interesting and important to analyze the related subgraph isomorphism problem, where only a subgraph (say, chosen uniformly at random) in the original graph is correlated with a node-permuted version in the other graph.
        \item Similarly to \cite{GraphAl5}, it is quite important to devise efficient algorithms that achieve strong detection (rather than weak detection as we proposed in our paper).
    \end{enumerate}

\bibliographystyle{ieeetr}
\bibliography{bibfile}


\appendix
\section{Distributions in $\calD_3$} \label{sec:Appendix}
\subsection{Stable distributions}
The set of stable distributions is broad, encompassing many well-known distributions such as the Normal, Cauchy, L\'evy, among others. We recall the definition of a stable distribution.
\begin{definition}[Stable distribution]
    Let $X_1$ and $X_2$ be independent copies of a random variable $X$. Then, $X$ is said to be stable if for any constants $\alpha,\beta>0$ the combination $\alpha X_1+\beta X_2$ has the same distribution as $\gamma X+\nu$, for some constants $\gamma>0$ and $\nu\in\mathbb{R}$. 
\end{definition}
Let ${A}$ and ${W}$ be two independent stable random variables with identical parameters $\alpha$, $\beta$, $\mu$, and $\gamma$, where $\alpha$ is the stability parameter, $\beta$ is the skewness parameter, $\mu$ is the location parameter, and $\gamma$ is the scale parameter. For simplicity, we assume $\mu =0$, though the results can be easily extended to the general case where $\mu$ is nonzero.
Define ${B} \triangleq c_1{A} + c_2{W}$. By the properties of stable distributions, ${B}$ is also stable with the same stability and skewness parameters, $\alpha$ and $\beta$. The scale parameter becomes $(\abs{c_1}^{\alpha}+\abs{c_2}^{\alpha})^{1/\alpha}\gamma$, whereas the location parameter remains $(c_1+c_2)\mu = 0$. To ensure that ${A}$ and ${B}$ have identical marginal distributions, we impose the condition $\abs{c_1}^{\alpha}+\abs{c_2}^{\alpha}=1$.
Next, we show that the correlation coefficient between ${A}$ and ${B}$ is  $c_1$. Indeed, due to the independence of ${A}$ and ${W}$, we have
\begin{align}
    \rho \triangleq \s{corr}(\calP, \calQ) 
    &= \frac{\s{cov}({A},{B})}{\Var({A})} 
    = \frac{\E\pp{{A}{B}}}{\Var({A})}
    = \frac{\E\pp{{A}(c_1 {A} + c_2 {W})}}{\Var\p{{A}}}
    = \frac{c_1A^2}{\Var({A})}
    = c_1.
\end{align}
Thus, we set $c_1 =\rho$, and $c_2 = \pm (1 - \abs{\rho}^{\alpha})^{1/\alpha}$. 
We now compute the conditional expectation,
\begin{align}
    \E\pp{{B}^{\ell} \vert {A}=a} 
    &= \E\pp{(\rho{A} + (1-\rho){W})^{\ell} \vert {A}=a} \\
    &= \E\pp{\sum_{k=0}^{\ell} \binom{\ell}{k}(\rho {A})^{\ell-k}\p{(1-\rho) {W}}^{k} \vert {A}=a} \\
    &= \rho^{\ell}a^\ell + \sum_{k=1}^{\ell} \binom{\ell}{k} \rho^{\ell-k}(1-\rho)^k a^{\ell-k}\E\pp{{W}^k}.
\end{align}
Since $W$ is independent of $A$ and has parameters $\alpha$, $\beta$, $\mu$ and $\gamma$ that are constant w.r.t. $a$, its moments $\E\pp{W^k}$ are also constant w.r.t $a$ for all $k$. As a result, $\E\pp{{B}^{\ell} \vert {A}=a}$ is a polynomial of degree $\ell$ in $a$ with a leading coefficient of $\s{corr}(\calP,\calQ)^{\ell}$. Therefore, the class $\calD_3$ properties are satisfied for the stable family of distributions considered.

\subsection{Bernoulli distribution} \label{App:Bernoulli}
Recall the definition of the Bernoulli case in Example \ref{exmp:2}; in this case, we have ${A}\sim\s{Bern}(ps)$, and, 
\begin{align}
    {B}\vert A \triangleq \begin{cases}
        \s{Bern}(s), \quad &{A}=1\\
        \s{Bern}\p{\frac{ps(1-s)}{1-ps}}, \quad &{A}=0.
    \end{cases}
\end{align}
It is not difficult to check that the marginal distribution of ${B}$ is identical to ${A}$. 
Furthermore, a straightforward calculation reveals that, 
\begin{align}
    \rho \triangleq \s{corr}(\calP,\calQ) 
    &= \frac{s(1-p)}{1-ps}. \label{eq:corr_Bernoulli}
\end{align} 
Now, for $\ell=0$, we trivially have that $\E\pp{{B}^0 \vert {A}} = 1 \cdot a^0$, and the leading coefficient is $1 = \s{corr}(\calP,\calQ)^0$. For $\ell\ne 0$, we have
\begin{align}
    \E\pp{{B}^{\ell} \vert {A}=a} 
    &= \begin{cases}
        \E\pp{\s{Bern}(s)}^{\ell}, &a=1\\ 
        \E\pp{\s{Bern}\p{\frac{ps(1-s)}{1-ps}}}^{\ell}, &a=0
    \end{cases} \\
    &= \begin{cases}
        s, &a=1\\ 
        \frac{ps(1-s)}{1-ps}, &a=0.
    \end{cases}
\end{align}
Since $a$ is either 0 or 1, we have that $a^\ell$ is also zero or one, respectively. Therefore, the projection $\E\pp{{B}^{\ell} \vert {A}=a}$ is a polynomial of degree $\ell$ in $a$ for both cases $a=0$ and $a=1$, 
with leading coefficients, $\kappa_{\ell} = 0$ for all $\ell \ge 1$ and $\kappa_0 = 1$. Therefore, the Bernoulli distribution above indeed falls inside $\calD_3$, since for $\ell=0$ we have $\kappa_0 = 1 = \s{corr}(\calP,\calQ)^0$, and for all $\ell \ge 1$ we have $0 \le \abs{\s{corr}(\calP,\calQ)}^\ell$.

\subsection{Poisson distribution}
Let ${A}$ be a Poisson random variable with parameter $\lambda$. We define,
\begin{align}
    {B} \triangleq {N} + {M},
\end{align}
where ${N}|A \sim \s{Binomial}({A},\rho)$ and ${M} \sim \s{Poisson}((1-\rho)\lambda)$, with ${M}$ independent of $(A,N)$. 
To show that $B$ follows a Poisson distribution with parameter $\lambda$, we compute its moment generating function (MGF). For $s\in\mathbb{R}$, we have,
\begin{align}
    \E\pp{s^{{B}}} &= \E\pp{\E\pp{s^{{N} + {M}}\vert {A}}} \\
    &= \E\pp{\E\pp{s^{{N}} \vert {A}} \E\pp{s^{{M}}}}\\
    &= \E\pp{(\rho s + 1 - \rho)^{{A}}} e^{(1-\rho)\lambda (s-1)}\\
    &= e^{\lambda(\rho s + 1 -\rho - 1)} e^{(1-\rho)\lambda (s-1)} \\
    &= e^{\lambda (s-1)}, \label{eq:B_poisson}
\end{align}
where the second equality follows from the independence of ${M}$ in $(A,N)$, and the third and fourth equalities follow by using the MGFs of the binomial and Poisson random variables. 
Therefore, \eqref{eq:B_poisson} 
matches the MGF of a Poisson distribution with parameter $\lambda$, confirming that ${B}\sim \s{Poisson}(\lambda)$. 
Next, we show that $\E\pp{{B}^{\ell}\vert {A}=a}$ is a polynomial of degree $\ell$ in $a$. Indeed,
\begin{align}
    \E\pp{{B}^{\ell}\vert {A}=a} 
    &= \E\pp{({N} + {M})^{\ell}\vert {A} = a}\\
    &= \sum\limits_{k=0}^{\ell} \binom{l}{k}\E\pp{{N}^k\vert {A}=a} \E\pp{{M}^{\ell-k}}. 
\end{align}
Since $M\sim\s{Poisson}((1-\rho)\lambda)$, the term $\E\pp{{M}^{\ell-k}}$ is a constant w.r.t. $a$. 
Furthermore, the term $\E\pp{{N}^k \vert {A}=a}$ is the $k$th moment of the Binomial random variable $({N} \vert {A}=a)$, and  can be expressed as
\begin{align}
    \E\pp{{N}^k\vert {A}=a} &= \sum\limits_{i=0}^k S(k,i) \rho^i a^i,
\end{align}
where $\{S(k,i)\}$ are the Stirling numbers of the second kind. Putting together, we conclude that $\E\pp{{B}^{\ell}\vert {A}=a}$ is a polynomial of degree $\ell$ in $a$ with a leading coefficient 
\begin{align}
    \kappa_{\ell} = \binom{\ell}{\ell} S(\ell,\ell) \rho^{\ell} = \rho^{\ell}.
\end{align}
Finally, the correlation coefficient between $A$ and $B$ is given by 
\begin{align}
    \s{corr}(\calP,\calQ) 
    &= \frac{\s{cov}\p{{A},{B}}}{\Var(\s{A})}  \label{eq:corr_calc_first}\\
    &= \frac{\s{cov}\p{{A},{N}}}{\Var({A})}\\
    &= \frac{\E\pp{{A}\E\pp{{N} \vert {A}}}- \E\pp{{A}} \E\pp{\E\pp{{N} \vert {A}}}}{\Var({A})}\\
    &= \frac{\rho\E\pp{{A}^2}- \rho\E^2\pp{{A}}}{\Var({A})}\\
    &= \rho \label{eq:corr_calc_last},
\end{align}
where the second equality follows the independence of $A$ and $M$.
Consequently, $\kappa_{\ell} = \rho^{\ell} = \s{corr}(\calP,\calQ)^{\ell}$, confirming that the Poisson distribution described is indeed in the class $\calD_3$.

\subsection{Chi-square distribution}
Let $A\sim\chi^2_k$ be a Chi-square random variable with parameter $k$. Define,
\begin{align}
    {B} \triangleq {W} + {Z},
\end{align}
where ${Z} \sim \chi^2_{(1-\rho)k}$ independent of ${A}$, and ${W}=\rho {A}$. It is well-known that the sum of two independent Chi-square distributed random variables $\chi^2_{k_1}$ and $\chi^2_{k_2}$ is distributed as $\chi^2_{k_1+k_2}$. Thus, ${B}={W}+{Z} \sim \chi^2_{k}$. Next, we show that the projection $\E\pp{{B}^{\ell}\vert {A}=a}$ is a polynomial of degree $\ell$ in $a$. Indeed,
\begin{align}
    \E\pp{{B}^{\ell}\vert {A}=a} 
    &= \E\pp{({W} + {Z})^{\ell}\vert {A} = a}\\
    &= \sum\limits_{k=0}^{\ell} \binom{l}{k}\E\pp{{W}^k\vert {A}=a} \E\pp{{Z}^{\ell-k}}\\
    &= \sum\limits_{k=0}^{\ell} \binom{l}{k}\rho^{k}a^k \E\pp{{Z}^{\ell-k}}. 
\end{align}
Note that since ${Z}\sim\chi^2_{(1-\rho)k}$, the term $\E\pp{{Z}^{\ell-k}}$ is a constant w.r.t. $a$. Hence, $\E\pp{{B}^\ell \vert {A}=a}$ is a polynomial of degree $\ell$ in $a$, with a leading coefficient,
\begin{align}
    \kappa_{\ell} = \binom{\ell}{\ell} \rho^{\ell}\E\pp{{Z}^0} =  \rho^{\ell},
\end{align}
It can be verified that the correlation coefficient $\s{corr}(A,B) = \rho$ in a manner analogous to the calculations in \eqref{eq:corr_calc_first}-\eqref{eq:corr_calc_last}. 
Thus, we have $\kappa_{\ell} = \rho^{\ell} = \s{corr}(\calP,\calQ)^{\ell}$, which demonstrates that the $\chi^2$ distribution described belongs to the class $\calD_3$.

\end{document}